\documentclass[final]{hld2026}

\usepackage[T1]{fontenc}
\usepackage[utf8]{inputenc}
\usepackage{mathtools}
\usepackage{amsfonts}
\usepackage{amsbsy}
\usepackage{bm}
\usepackage{mathrsfs}
\usepackage{dsfont}
\usepackage{bbm}
\usepackage{booktabs}
\usepackage[table]{xcolor}
\usepackage{enumitem}
\usepackage{microtype}
\usepackage{float}
\hypersetup{
  pdftitle={A Coulomb Particle Model for Learning Kernel Attention in Transformers},
  pdfauthor={Masoud Badiei Khuzani, Atiq Islam, Alex Cozzi, Abraham Bagherjeiran},
  hypertexnames=false
}
\title[Learning Kernel Attention]{A Coulomb Particle Model for Learning Kernel Attention in Transformers}

\hldauthor{%
\Name{Masoud Badiei Khuzani} \Email{mbadieikhuzani@ebay.com}\\
\addr {eBay Inc., San Jose, CA, USA}
\AND
\Name{Sharath Honnaiah}
 \Email{shonnaiah@ebay.com}\\
\addr {eBay Inc., San Jose, CA, USA}
\AND
\Name{Atiq Islam} \Email{atislam@ebay.com}\\
\addr {eBay Inc., San Jose, CA, USA}
\AND
\Name{Alex Cozzi} \Email{acozzi@ebay.com}\\
\addr {eBay Inc., San Jose, CA, USA}
\AND
\Name{Abraham Bagherjeiran} \Email{abagherjeiran@ebay.com}\\
\addr {eBay Inc., San Jose, CA, USA}}
\numberwithin{equation}{section}

\makeatletter
\newcommand{\hild@maybe@note}[1]{%
  \if\relax\detokenize{#1}\relax\else\ (#1)\fi%
}
\@ifundefined{c@theorem}{%
  \newcounter{theorem}[section]
  \renewcommand{\thetheorem}{\thesection.\arabic{theorem}}
}{}
\@ifundefined{theorem}{%
  \newenvironment{theorem}[1][]{%
    \refstepcounter{theorem}%
    \par\medskip\noindent\textbf{Theorem~\thetheorem\hild@maybe@note{#1}.}\itshape\ }%
    {\par\medskip}%
}{}
\@ifundefined{lemma}{%
  \newenvironment{lemma}[1][]{%
    \refstepcounter{theorem}%
    \par\medskip\noindent\textbf{Lemma~\thetheorem\hild@maybe@note{#1}.}\itshape\ }%
    {\par\medskip}%
}{}
\@ifundefined{proposition}{%
    {\par\medskip}%
}{}
\@ifundefined{corollary}{%
    {\par\medskip}%
}{}
\@ifundefined{definition}{%
  \newenvironment{definition}[1][]{%
    \refstepcounter{theorem}%
    \par\medskip\noindent\textbf{Definition~\thetheorem\hild@maybe@note{#1}.}\ }%
    {\par\medskip}%
}{}
\@ifundefined{example}{%
    {\par\medskip}%
}{}
\@ifundefined{remark}{%
    {\par\medskip}%
}{}
\@ifundefined{c@assumption}{%
  \newcounter{assumption}
  \renewcommand{\theassumption}{A.\arabic{assumption}}
}{}
\@ifundefined{assumption}{%
  \newenvironment{assumption}[1][]{%
    \refstepcounter{assumption}%
    \par\medskip\noindent\textbf{Assumption~\theassumption\hild@maybe@note{#1}.}\itshape\ }%
    {\par\medskip}%
}{}
\@ifundefined{proof}{%
  \newenvironment{proof}[1][Proof]{%
    \par\medskip\noindent\textit{#1.}\ }%
    {\hfill$\square$\par\medskip}%
}{}
\makeatother

\DeclareMathOperator{\Ent}{Ent}
\DeclareMathOperator{\Law}{Law}

\newcommand{\real}{\mathbb{R}}

\newcommand{\prob}{\mathbb{P}}
\newcommand{\expect}{\mathbb{E}}

\newcommand{\defeq}{\stackrel{\mathrm{def}}{=}}

\newcommand{\safeincludegraphics}[2][]{%
  \IfFileExists{#2}{\includegraphics[#1]{#2}}{%
    \fbox{\parbox{0.86\linewidth}{\centering Missing figure file: \texttt{\detokenize{#2}}}}%
  }%
}

\begin{document}
\maketitle

\begin{abstract}%
Randomized features provide a scalable approximation to kernel machines, but their performance depends strongly on the choice of feature distribution. We propose a particle-based method that learns this distribution by optimizing kernel-target alignment while regularizing particles with a Riesz/Coulomb repulsive potential. The resulting Hamiltonian yields diverse, task-adaptive random features and admits a mean-field description through a McKean--Vlasov equation. We instantiate the method in linearized Transformer attention by learning positive random-feature maps in a first alignment phase, then freezing the kernel and training the remaining network parameters with cross-entropy. Experiments on synthetic classification and sentence-level benchmarks show that learned kernelized attention can improve accuracy, calibration, and robustness for several feature maps while preserving linear-attention inference complexity.
\end{abstract}

\section{Introduction}
Kernel methods are a principled way to encode nonlinear similarity, but classical training and prediction can scale poorly with the number of samples. Random Fourier features alleviate this problem by replacing an implicit kernel with an explicit randomized map \citep{rahimi2007random}. Nevertheless, the kernel and its associated feature distribution are typically fixed before any labels are observed—a choice that can dominate downstream performance when the appropriate similarity structure is unknown.

We study a supervised kernel-learning procedure that keeps the computational advantages of random features while learning the feature distribution from labels. The key idea is to view the random features as interacting particles. A target-alignment objective attracts particles toward features that explain the labels, while a Riesz/Coulomb repulsive energy prevents collapse and encourages diversity. This gives a concrete optimization algorithm, a statistical-mechanics interpretation, and a direct route to kernelized attention in Transformers.

Our contributions are: (i) a particle Hamiltonian for target-aligned random-feature learning; (ii) a Langevin optimization procedure with a mean-field continuum limit; (iii) an instantiation for positive random-feature attention that preserves linear-time inference; and (iv) empirical evidence that the learned kernels improve several accuracy and calibration metrics on synthetic and NLP benchmarks. Full related work, proofs, and additional ablations are in the supplementary material.

\section{Particle kernel learning}
Let \(\phi:\mathcal{X}\times\Omega\to[-1,1]\) be a random feature map and let \(\mu\) be a probability measure over feature parameters. The induced kernel is
\begin{align}
K_{\mu}(\bm{x},\bm{x}')
=\expect_{\bm{\omega}\sim\mu}[\phi_{\bm{\omega}}(\bm{x})\phi_{\bm{\omega}}(\bm{x}')].
\end{align}
Given training data \(\{(\bm{x}_i,y_i)\}_{i=1}^n\), we learn \(\mu\) by maximizing kernel-target alignment \citep{cristianini2002kernel}. With particles \(\bm{\Omega}_N=(\bm{\omega}_1,\ldots,\bm{\omega}_N)\) and empirical measure \(\mu_N=N^{-1}\sum_{k=1}^N\delta_{\bm{\omega}_k}\), this becomes the finite-dimensional Hamiltonian
\begin{align}
\label{eq:main-hamiltonian}
\mathcal{H}_N(\bm{\Omega}_N)
= -\frac{1}{n(n-1)}\sum_{i\ne j}y_i y_j\frac{1}{N}\sum_{k=1}^N
\phi_{\bm{\omega}_k}(\bm{x}_i)\phi_{\bm{\omega}_k}(\bm{x}_j)
+\lambda\mathcal{W}_{N,s}(\bm{\Omega}_N),
\end{align}
where
\begin{align}
\mathcal{W}_{N,s}(\bm{\Omega}_N)
=\frac{1}{2N(N-1)}\sum_{k\ne \ell}g_s(\bm{\omega}_k-\bm{\omega}_\ell),\qquad
 g_s(\bm{\omega})=
\begin{cases}
\|\bm{\omega}\|_2^{-s}, & s>0,\\
-\log\|\bm{\omega}\|_2, & s=0.
\end{cases}
\end{align}
The first term rewards features whose empirical kernel aligns with labels; the second term spreads particles across \(\Omega\). For random Fourier features, this energy reduces to a trigonometric alignment objective involving \(\cos(\bm{X}\bm{\Omega}_N^\top)\) and \(\sin(\bm{X}\bm{\Omega}_N^\top)\), which makes the objective differentiable and easy to optimize.

\begin{algorithm2e}[!t]
    \caption{Projected Langevin feature learning}
    \label{alg:main-langevin}
    \scriptsize
    \DontPrintSemicolon
    \KwData{particles \(\{\bm\omega_k^0\}_{k=1}^N\), step \(\eta\), inverse temperature \(\beta\), threshold \(\delta\)}
    \KwResult{learned empirical feature law \(\mu_N\) and kernel \(\widehat K\)}
    Initialize \(\mu_N^0=N^{-1}\sum_k\delta_{\bm\omega_k^0}\)\;
    \Repeat{\(\mathcal D(\mu_N^m,\mu_N^{m-1})<\delta\)}{
        \For{\(k=1,\ldots,N\)}{
            Draw \(\bm\xi_k^m\sim\mathcal N(0,I)\) and set
            \(\bm\omega_k^{m+1}=\mathcal P_\Omega(\bm\omega_k^m-\eta N\nabla_{\bm\omega_k}\mathcal H_N(\bm\Omega_N^m)+\sqrt{2\eta/\beta}\,\bm\xi_k^m)\)\;
        }
        \(\mu_N^{m+1}=N^{-1}\sum_k\delta_{\bm\omega_k^{m+1}}\)\;
    }
    Return \(\widehat K(\bm{x},\bm{x}')=D^{-1}\sum_{k=1}^D\phi_{\bm\omega_k}(\bm{x})\phi_{\bm\omega_k}(\bm{x}')\)\;
\end{algorithm2e}

\section{Kernelized attention}
Self-attention can be written as a normalized kernel smoother \citep{vaswani2017attention,nadaraya1964,watson1964}:
\begin{align}
\bm{a}_i(\bm{X})
=\frac{\sum_{j=0}^{\ell}K(\bm{q}_i,\bm{k}_j)\bm{v}_j}{\sum_{j=0}^{\ell}K(\bm{q}_i,\bm{k}_j)},
\qquad
K(\bm{q},\bm{k})=\exp(\bm{q}^\top\bm{k}/\sqrt d).
\end{align}
Replacing \(K\) by a positive random-feature kernel \(K_{\bm{\Omega}_N}(\bm{q},\bm{k})=\bm{\phi}_{\bm{\Omega}_N}(\bm{q})^\top\bm{\phi}_{\bm{\Omega}_N}(\bm{k})\) gives the linearized attention estimator
\begin{align}
\widehat{\bm{a}}_i(\bm{X})
=\frac{\bm{\phi}_{\bm{\Omega}_N}(\bm{q}_i)^\top\big(\sum_j\bm{\phi}_{\bm{\Omega}_N}(\bm{k}_j)\bm{v}_j^\top\big)}{\bm{\phi}_{\bm{\Omega}_N}(\bm{q}_i)^\top\big(\sum_j\bm{\phi}_{\bm{\Omega}_N}(\bm{k}_j)\big)}.
\end{align}
This preserves the normalized kernel-smoothing form but reduces the per-head sequence-length dependence from quadratic to linear in \(\ell\), up to the feature dimension \citep{katharopoulos2020linear,choromanski2021performer,peng2021rfa}. In Phase A, we learn the feature particles \(\bm{\Omega}_N\) by alignment on sequence representations. In Phase B, we freeze \(\bm{\Omega}_N\) and train the remaining Transformer parameters using cross-entropy.

\section{Theoretical results: mean-field limit and large deviation principle} The particle view gives both a continuum training dynamics and an equilibrium concentration result.  Define the alignment-induced potential
\begin{align}
V_{\mathcal D}(\bm{\omega})
= -\frac{1}{n(n-1)}\sum_{i\ne j} y_i y_j
\phi_{\bm{\omega}}(\bm{x}_i)\phi_{\bm{\omega}}(\bm{x}_j),
\end{align}
and the continuum energy
\begin{align}
\label{eq:continuum-energy-main}
\mathcal{E}_s(\mu)=\int V_{\mathcal D}(\bm{\omega})\,d\mu(\bm{\omega})
+\frac{\lambda}{2}\iint g_s(\bm{\omega}-\bm{\omega}')\,d\mu(\bm{\omega})d\mu(\bm{\omega}').
\end{align}
The finite-particle optimizer has the following continuum training law.
\begin{theorem}[Projected-particle McKean--Vlasov mean-field limit]
    \label{thm:main-mckean-vlasov}
    Let \(\Omega\subset\mathbb R^d\) be compact and convex with \(C^2\) boundary and outward normal \(\bm n\). Assume \(V_{\mathcal D}\in C^2(\overline\Omega)\), a Lipschitz regularized interaction drift, exchangeable \(\mu_0^N\Rightarrow\rho_0d\bm\omega\), and a vanishing projected-Euler error. If
    \begin{align}
        \bm\omega_k^{m+1}=\mathcal P_{\overline\Omega}\!\left(\bm\omega_k^m-\eta_N N\nabla_{\bm\omega_k}\mathcal H_N(\bm\Omega_N^m)+\sqrt{2\eta_N/\beta}\,\bm\xi_k^m\right),\quad \eta_N\downarrow0,
        \label{eq:main-projected-mf-euler}
    \end{align}
    then \(\mu_t^N=N^{-1}\sum_k\delta_{\bm\omega_k^{\lfloor t/\eta_N\rfloor}}\Rightarrow\rho_t d\bm\omega\) in probability, uniformly on compact time intervals. With \(U_t=V_{\mathcal D}+\lambda g_s\ast\rho_t\),
    \begin{align}
        \partial_t\rho_t=\nabla\!\cdot(\rho_t\nabla U_t)+\beta^{-1}\Delta\rho_t,\\
        \rho_t|_{t=0}=\rho_0,
        \qquad
        (\rho_t\nabla U_t+\beta^{-1}\nabla\rho_t)\!\cdot\bm n=0\quad\text{on }\partial\Omega,
        \label{eq:main-mckean-vlasov}
    \end{align}
    equivalently \(\partial_{\bm n}\rho_t+\beta\rho_t\partial_{\bm n}U_t=0\). Moreover \(\rho_t\ge0\) and \(\int_\Omega\rho_t=1\).
\end{theorem}

The proof is provided in the supplementary material. The two parts of \(U_t\) have distinct roles.  The data potential \(V_{\mathcal D}\) pulls mass toward features that reduce the empirical objective, while the interaction term spreads mass according to the regularized Riesz geometry and prevents all features from collapsing onto the same locations.  The diffusion term \(\beta^{-1}\Delta\rho_t\) is the continuum trace of the Langevin noise: at finite temperature it encourages exploration and contributes an entropic regularization; in the zero-temperature limit the equation reduces to the deterministic transport law driven by \(-\nabla U_t\).  Equivalently,
\begin{align*}
    \partial_t\rho_t
    =\nabla\!\cdot\left(\rho_t\nabla\left(U_t+\beta^{-1}\log\rho_t\right)\right),
\end{align*}
so \eqref{eq:main-mckean-vlasov} is the Wasserstein gradient flow of $    \mathcal F_\beta(\rho)
\defeq\mathcal E_s(\rho)+\beta^{-1}\Ent(\rho\mid\ell).$
For smooth positive solutions the no-flux condition removes the boundary contribution and gives the dissipation identity
\begin{align*}
    \frac{d}{dt}\mathcal F_\beta(\rho_t)
    =-\int_\Omega \rho_t\left|\nabla\left(U_t+\beta^{-1}\log\rho_t\right)\right|^2d\bm\omega\le0.
\end{align*}
This identity is useful conceptually: training decreases the continuum free energy, and the only stationary points are self-consistent Gibbs densities of the form
\begin{align*}
    \rho_\infty(\bm\omega)
    =\frac{1}{Z_\infty}\exp\{-\beta[V_{\mathcal D}(\bm\omega)+\lambda(g_s\ast\rho_\infty)(\bm\omega)]\},
\end{align*}
with the boundary condition inherited from projection.

\begin{theorem}[Large deviations for learned feature measures]
\label{thm:main-ldp}
Let \(\Omega\subset\mathbb{R}^d\) be bounded and let \(s>d\).  If \(\bm{\Omega}_N\sim \prob_{N,\beta_N}\) and \(\mu_N=N^{-1}\sum_{k=1}^N\delta_{\bm{\omega}_k}\), then \(\{\mu_N\}\) satisfies an LDP on \(\mathcal{P}(\Omega)\).  When \(\beta_N/N\to1\), the speed is \(r_N=N\) and
\begin{align}
\label{eq:rate-thermal-main}
\mathcal{J}_s(\mu)=
\mathcal{E}_s(\mu)+\Ent(\mu\mid\ell)
-\inf_{\nu}\{\mathcal{E}_s(\nu)+\Ent(\nu\mid\ell)\},
\end{align}
where \(\Ent(\mu\mid\ell)=\int\log(d\mu/d\ell)d\mu\) for \(\mu\ll\ell\).  When \(\beta_N/N\to\infty\), the speed is \(r_N=\beta_N\) and
\begin{align}
\label{eq:rate-zero-temp-main}
\mathcal{J}_s(\mu)=\mathcal{E}_s(\mu)-\inf_{\nu}\mathcal{E}_s(\nu).
\end{align}
For Borel \(A\subset\mathcal{P}(\Omega)\), the standard LDP hold with rate \(\mathcal{J}_s\) and speed \(r_N\).
\end{theorem}
The supplement proves this by combining Sanov's theorem with Varadhan's lemma after truncating the singular Riesz interaction. Consequently the learned kernel concentrates around the task-adaptive variational kernel; if the relevant rate has a unique minimizer \(\mu_s^\star\), then \(K_{\mu_N}(\bm x,\bm x')\to K_{\mu_s^\star}(\bm x,\bm x')\) exponentially at speed \(N\) or \(\beta_N\).

\section{Experiments}
We evaluate on synthetic classification and SST-2, QQP, and Rotten Tomatoes.  Synthetic results show that optimized particles outperform fixed random Fourier features and the importance-sampling baseline under moderate noise. The accuracy gap widens with feature budget: at 64 features per head, optimized particles retain roughly 90\% of the performance achieved at 256, while fixed random features degrade more sharply. NLP experiments use a two-layer Transformer encoder with hidden size 128, two heads, 256 random features per head, BERT tokenization, batch size 64, and the two-phase alignment-then-cross-entropy protocol. Training uses AdamW with linear warmup and cosine decay; the alignment phase runs for 20\% of total steps before switching to cross-entropy fine-tuning. All results are averaged over three random seeds.
\vspace{-1mm}
We report the full detailed result tables for the three text-classification tasks in the main paper. For all rows in the lower block of each table, the attention projection is constrained to \(Q=K=\mathrm{reshape}(x)\), so vanilla and learned-kernel rows are matched feature-map comparisons. The reported metrics include accuracy, micro/macro/weighted F1, ROC-AUC, PR-AUC, MCC, balanced accuracy, LogLoss, Brier score, ECE, and wall-clock training time.
\vspace{-1mm}
\begin{table}[H]
    \centering
    \tiny
    \renewcommand{\arraystretch}{0.68}
    \setlength{\tabcolsep}{2.0pt}
    \setlength{\abovecaptionskip}{1pt}
    \setlength{\belowcaptionskip}{0pt}
    \caption{SST-2 detailed metrics across attention variants. Bold marks the best value per column; lower is better for LogLoss, Brier, ECE, and train time.}
    \label{tab:sst2_detailed_metrics}
    \resizebox{\textwidth}{!}{%
        \begin{tabular}{lcccccccccccc}
            \toprule
            \textbf{Model} & \textbf{Acc}$\uparrow$ & \textbf{F1}$\bm{\mu}\uparrow$ & \textbf{F1M}$\uparrow$ & \textbf{F1w}$\uparrow$ & \textbf{ROC-AUC}$\uparrow$ & \textbf{PR-AUC}$\uparrow$ & \textbf{MCC}$\uparrow$ & \textbf{BalAcc}$\uparrow$ & \textbf{LogLoss}$\downarrow$ & \textbf{Brier}$\downarrow$ & \textbf{ECE}$\downarrow$ & \textbf{Train(sec)}$\downarrow$ \\
            \midrule
            
            \rowcolor{gray!15}
            \multicolumn{13}{c}{\tiny\textbf{Learned $W_q, W_k$ baselines}} \\
            \addlinespace[2pt]
            Vaswani-softmaxattn & 0.8050 & 0.8050 & 0.8050 & 0.8051 & 0.8744 & 0.8856 & 0.6100 & 0.8050 & 0.5083 & 0.1498 & 0.3695 & 86.8 \\
            Vanilla-favor-WqWk   & 0.7833 & 0.7833 & 0.7832 & 0.7832 & 0.8587 & 0.8576 & 0.5676 & 0.7837 & 0.5472 & 0.1592 & 0.3584 & 96.5 \\
            Performer            & 0.7798 & 0.7798 & 0.7793 & 0.7791 & 0.8618 & 0.8766 & 0.5656 & 0.7810 & 0.5697 & 0.1666 & 0.3804 & 89.0 \\
            Linformer            & 0.7844 & 0.7844 & 0.7843 & 0.7843 & 0.8569 & 0.8673 & 0.5706 & 0.7850 & 0.5922 & 0.1646 & 0.3782 & 82.6 \\
            \addlinespace[4pt]
            \midrule
            
            \rowcolor{gray!15}
            \multicolumn{13}{c}{\tiny\textbf{No learned $W_q, W_k$ (Q=K=reshape(x)); feature-map variants}} \\
            \addlinespace[2pt]
            Vanilla-favor         & 0.8050 & 0.8050 & 0.8050 & 0.8049 & 0.8841 & 0.8928 & 0.6124 & 0.8057 & \textbf{0.4362} & 0.1388 & 0.3420 & 80.5 \\
            Kernel-favor          & 0.7901 & 0.7901 & 0.7900 & 0.7901 & 0.8734 & 0.8834 & 0.5801 & 0.7899 & 0.4477 & 0.1459 & \textbf{0.3086} & 166.0 \\
            Vanilla-elu           & 0.8028 & 0.8028 & 0.8019 & 0.8017 & 0.8731 & 0.8803 & 0.6150 & 0.8042 & 0.4938 & 0.1505 & 0.3724 & \textbf{78.8} \\
            Kernel-elu            & 0.8108 & 0.8108 & 0.8107 & 0.8108 & 0.8879 & 0.8913 & 0.6216 & 0.8108 & 0.4502 & 0.1392 & 0.3609 & 164.3 \\
            Vanilla-softplus      & 0.8108 & 0.8108 & 0.8108 & 0.8108 & 0.8760 & 0.8837 & 0.6225 & 0.8112 & 0.4582 & 0.1429 & 0.3378 & 88.5 \\
            Kernel-softplus       & 0.8062 & 0.8062 & 0.8058 & 0.8060 & 0.8861 & 0.8864 & 0.6126 & 0.8057 & 0.5122 & 0.1470 & 0.3920 & 173.5 \\
            Vanilla-sigmoid2      & 0.7959 & 0.7959 & 0.7957 & 0.7958 & 0.8786 & 0.8881 & 0.5916 & 0.7956 & 0.5025 & 0.1487 & 0.3744 & 82.6 \\
            Kernel-sigmoid2       & 0.7924 & 0.7924 & 0.7924 & 0.7924 & 0.8785 & 0.8880 & 0.5863 & 0.7929 & 0.4915 & 0.1516 & 0.3790 & 170.1 \\
            Vanilla-softmaxfeat   & 0.8005 & 0.8005 & 0.8003 & 0.8004 & 0.8683 & 0.8728 & 0.6008 & 0.8001 & 0.5163 & 0.1530 & 0.3621 & 89.5 \\
            Kernel-softmaxfeat    & 0.8085 & 0.8085 & 0.8085 & 0.8084 & 0.8858 & 0.8982 & 0.6181 & 0.8089 & 0.4381 & 0.1389 & 0.3526 & 171.4 \\
            Vanilla-cos2          & 0.8142 & 0.8142 & 0.8142 & 0.8141 & \textbf{0.8897} & 0.8945 & 0.6305 & 0.8148 & 0.4709 & 0.1395 & 0.3866 & 81.5 \\
            Kernel-cos2           & 0.8119 & 0.8119 & 0.8118 & 0.8119 & 0.8883 & 0.8976 & 0.6237 & 0.8118 & 0.4422 & \textbf{0.1366} & 0.3557 & 165.0 \\
            Vanilla-porf-softplus & 0.8142 & 0.8142 & 0.8142 & 0.8142 & 0.8883 & \textbf{0.8983} & 0.6285 & 0.8143 & 0.4506 & 0.1380 & 0.3569 & 85.7 \\
            Kernel-porf-softplus  & \textbf{0.8188} & \textbf{0.8188} & \textbf{0.8188} & \textbf{0.8188} & \textbf{0.8897} & 0.8967 & \textbf{0.6376} & \textbf{0.8189} & 0.4561 & 0.1381 & 0.3734 & 162.8 \\
            \bottomrule
        \end{tabular}%
    }
\end{table}
\vspace{-3mm}
\begin{table}[H]
    \centering
    \tiny
    \renewcommand{\arraystretch}{0.68}
    \setlength{\tabcolsep}{2.0pt}
    \setlength{\abovecaptionskip}{1pt}
    \setlength{\belowcaptionskip}{0pt}
    \caption{QQP detailed metrics across attention variants. Bold marks the best value per column; lower is better for LogLoss, Brier, ECE, and train time.}
    \label{tab:qqp_detailed_metrics}
    \resizebox{\textwidth}{!}{%
        \begin{tabular}{lcccccccccccc}
            \toprule
            \textbf{Model} & \textbf{Acc}$\uparrow$ & \textbf{F1}$\bm{\mu}\uparrow$ & \textbf{F1M}$\uparrow$ & \textbf{F1w}$\uparrow$ & \textbf{ROC-AUC}$\uparrow$ & \textbf{PR-AUC}$\uparrow$ & \textbf{MCC}$\uparrow$ & \textbf{BalAcc}$\uparrow$ & \textbf{LogLoss}$\downarrow$ & \textbf{Brier}$\downarrow$ & \textbf{ECE}$\downarrow$ & \textbf{Train(sec)}$\downarrow$ \\
            \midrule
            
            \rowcolor{gray!15}
            \multicolumn{13}{c}{\tiny\textbf{Learned $W_q, W_k$ baselines}} \\
            \addlinespace[2pt]
            Vaswani-softmaxattn & 0.8005 & 0.8005 & \textbf{0.7949} & 0.8038 & \textbf{0.8959} & \textbf{0.8266} & \textbf{0.6058} & \textbf{0.8133} & \textbf{0.4112} & 0.1350 & 0.4314 & 605.9 \\
            Vanilla-favor-WqWk  & 0.7683 & 0.7683 & 0.7346 & 0.7595 & 0.8313 & 0.7619 & 0.4852 & 0.7248 & 0.4863 & 0.1591 & 0.4383 & \textbf{539.7} \\
            Performer           & 0.7751 & 0.7751 & 0.7472 & 0.7693 & 0.8402 & 0.7691 & 0.5030 & 0.7388 & 0.4732 & 0.1547 & 0.4360 & 554.1 \\
            Linformer           & \textbf{0.8077} & \textbf{0.8077} & 0.7944 & \textbf{0.8082} & 0.8828 & 0.8157 & 0.5890 & 0.7960 & 0.4153 & \textbf{0.1337} & 0.4541 & 557.1 \\
            \addlinespace[4pt]
            \midrule
            
            \rowcolor{gray!15}
            \multicolumn{13}{c}{\tiny\textbf{No learned $W_q, W_k$ (Q=K=reshape(x)); feature-map variants}} \\
            \addlinespace[2pt]
            Vanilla-favor         & 0.7626 & 0.7626 & 0.7323 & 0.7560 & 0.8168 & 0.7443 & 0.4741 & 0.7242 & 0.5028 & 0.1637 & 0.4182 & 567.8 \\
            Kernel-favor          & 0.7608 & 0.7608 & 0.7319 & 0.7551 & 0.8158 & 0.7438 & 0.4712 & 0.7246 & 0.5016 & 0.1638 & \textbf{0.4106} & 1197.1 \\
            Vanilla-elu           & 0.7797 & 0.7797 & 0.7553 & 0.7757 & 0.8458 & 0.7772 & 0.5155 & 0.7485 & 0.4659 & 0.1518 & 0.4386 & 585.1 \\
            Kernel-elu            & 0.7831 & 0.7831 & 0.7648 & 0.7821 & 0.8533 & 0.7863 & 0.5299 & 0.7625 & 0.4598 & 0.1496 & 0.4474 & 1110.3 \\
            Vanilla-softplus      & 0.7786 & 0.7786 & 0.7522 & 0.7735 & 0.8442 & 0.7747 & 0.5116 & 0.7444 & 0.4656 & 0.1522 & 0.4297 & 633.3 \\
            Kernel-softplus       & 0.7860 & 0.7860 & 0.7651 & 0.7836 & 0.8564 & 0.7891 & 0.5323 & 0.7603 & 0.4578 & 0.1481 & 0.4582 & 1100.7 \\
            Vanilla-sigmoid2      & 0.7772 & 0.7772 & 0.7580 & 0.7760 & 0.8439 & 0.7728 & 0.5165 & 0.7554 & 0.4664 & 0.1529 & 0.4190 & 591.9 \\
            Kernel-sigmoid2       & 0.7852 & 0.7852 & 0.7666 & 0.7840 & 0.8549 & 0.7885 & 0.5338 & 0.7639 & 0.4537 & 0.1479 & 0.4415 & 1156.0 \\
            Vanilla-softmaxfeat   & 0.7801 & 0.7801 & 0.7527 & 0.7744 & 0.8471 & 0.7772 & 0.5143 & 0.7441 & 0.4625 & 0.1512 & 0.4365 & 579.9 \\
            Kernel-softmaxfeat    & 0.7909 & 0.7909 & 0.7753 & 0.7909 & 0.8679 & 0.7956 & 0.5506 & 0.7753 & 0.4354 & 0.1418 & 0.4371 & 1072.6 \\
            Vanilla-cos2          & 0.7746 & 0.7746 & 0.7533 & 0.7724 & 0.8422 & 0.7715 & 0.5082 & 0.7492 & 0.4688 & 0.1536 & 0.4249 & 650.2 \\
            Kernel-cos2           & 0.7842 & 0.7842 & 0.7590 & 0.7795 & 0.8525 & 0.7842 & 0.5245 & 0.7511 & 0.4621 & 0.1497 & 0.4526 & 1254.0 \\
            Vanilla-porf\_softplus & 0.7782 & 0.7782 & 0.7531 & 0.7739 & 0.8434 & 0.7751 & 0.5117 & 0.7461 & 0.4685 & 0.1528 & 0.4401 & 588.4 \\
            Kernel-porf\_softplus  & 0.7841 & 0.7841 & 0.7632 & 0.7817 & 0.8536 & 0.7855 & 0.5284 & 0.7585 & 0.4547 & 0.1483 & 0.4419 & 1110.4 \\
            \bottomrule
        \end{tabular}%
    }
\end{table}
\vspace{-3mm}
\begin{table}[H]
    \centering
    \tiny
    \renewcommand{\arraystretch}{0.68}
    \setlength{\tabcolsep}{2.0pt}
    \setlength{\abovecaptionskip}{1pt}
    \setlength{\belowcaptionskip}{0pt}
    \caption{Rotten Tomatoes detailed metrics across attention variants. Bold marks the best value per column; lower is better for LogLoss, Brier, ECE, and train time.}
    \label{tab:rotten_tomatoes_detailed_metrics}
    \resizebox{\textwidth}{!}{%
        \begin{tabular}{lcccccccccccc}
            \toprule
            \textbf{Model} & \textbf{Acc}$\uparrow$ & \textbf{F1}$\bm{\mu}\uparrow$ & \textbf{F1M}$\uparrow$ & \textbf{F1w}$\uparrow$ & \textbf{ROC-AUC}$\uparrow$ & \textbf{PR-AUC}$\uparrow$ & \textbf{MCC}$\uparrow$ & \textbf{BalAcc}$\uparrow$ & \textbf{LogLoss}$\downarrow$ & \textbf{Brier}$\downarrow$ & \textbf{ECE}$\downarrow$ & \textbf{Train(sec)}$\downarrow$ \\
            \midrule
            
            \rowcolor{gray!15}
            \multicolumn{13}{c}{\tiny\textbf{Learned $W_q, W_k$ baselines}} \\
            \addlinespace[2pt]
            Vaswani-softmaxattn & 0.6801 & 0.6801 & 0.6793 & 0.6793 & 0.7492 & 0.7522 & 0.3620 & 0.6801 & 0.6396 & 0.2138 & 0.2751 & 12.6 \\
            Vanilla-favor-WqWk  & 0.6764 & 0.6764 & 0.6758 & 0.6758 & 0.7262 & 0.7162 & 0.3539 & 0.6764 & 0.7265 & 0.2327 & 0.3052 & 13.0 \\
            Performer           & 0.6642 & 0.6642 & 0.6616 & 0.6616 & 0.7262 & 0.7075 & 0.3333 & 0.6642 & 0.7477 & 0.2347 & 0.3009 & 14.4 \\
            Linformer           & 0.6614 & 0.6614 & 0.6612 & 0.6612 & 0.7133 & 0.7062 & 0.3230 & 0.6614 & 0.6614 & 0.2260 & 0.2509 & 13.7 \\
            \addlinespace[4pt]
            \midrule
            
            \rowcolor{gray!15}
            \multicolumn{13}{c}{\tiny\textbf{No learned $W_q, W_k$ (Q=K=reshape(x)); feature-map variants}} \\
            \addlinespace[2pt]
            Vanilla-favor         & 0.7083 & 0.7083 & 0.7066 & 0.7066 & 0.7732 & 0.7708 & 0.4214 & 0.7083 & 0.5840 & 0.1977 & 0.2229 & 13.7 \\
            Kernel-favor          & 0.6848 & 0.6848 & 0.6848 & 0.6848 & 0.7597 & 0.7479 & 0.3696 & 0.6848 & 0.5823 & 0.1997 & 0.2193 & 26.3 \\
            Vanilla-elu           & 0.6895 & 0.6895 & 0.6835 & 0.6835 & \textbf{0.7844} & \textbf{0.7928} & 0.3943 & 0.6895 & 0.5809 & 0.1995 & 0.2435 & 12.7 \\
            Kernel-elu            & 0.7148 & 0.7148 & 0.7140 & 0.7140 & 0.7650 & 0.7586 & 0.4320 & 0.7148 & 0.5860 & 0.1991 & 0.2594 & 26.1 \\
            Vanilla-softplus      & 0.7036 & 0.7036 & 0.7023 & 0.7023 & 0.7695 & 0.7580 & 0.4107 & 0.7036 & 0.5838 & 0.1977 & 0.2176 & 12.6 \\
            Kernel-softplus       & 0.7017 & 0.7017 & 0.7017 & 0.7017 & 0.7743 & 0.7677 & 0.4034 & 0.7017 & 0.5766 & 0.1959 & 0.2517 & 23.8 \\
            Vanilla-sigmoid2      & 0.7026 & 0.7026 & 0.7022 & 0.7022 & 0.7714 & 0.7675 & 0.4065 & 0.7026 & 0.5758 & 0.1958 & 0.2408 & 12.7 \\
            Kernel-sigmoid2       & 0.6876 & 0.6876 & 0.6831 & 0.6831 & 0.7562 & 0.7538 & 0.3864 & 0.6876 & 0.6118 & 0.2091 & 0.2733 & 24.0 \\
            Vanilla-softmaxfeat   & 0.6782 & 0.6782 & 0.6733 & 0.6733 & 0.7664 & 0.7703 & 0.3678 & 0.6782 & 0.6065 & 0.2071 & 0.2610 & 12.5 \\
            Kernel-softmaxfeat    & \textbf{0.7167} & \textbf{0.7167} & \textbf{0.7163} & \textbf{0.7163} & 0.7779 & 0.7724 & \textbf{0.4345} & \textbf{0.7167} & \textbf{0.5697} & \textbf{0.1929} & 0.2527 & 24.5 \\
            Vanilla-cos2          & 0.6876 & 0.6876 & 0.6855 & 0.6855 & 0.7622 & 0.7586 & 0.3803 & 0.6876 & 0.5897 & 0.2015 & \textbf{0.2154} & \textbf{12.4} \\
            Kernel-cos2           & 0.7092 & 0.7092 & 0.7091 & 0.7091 & 0.7676 & 0.7680 & 0.4185 & 0.7092 & 0.5836 & 0.1982 & 0.2594 & 26.7 \\
            Vanilla-porf\_softplus & 0.6839 & 0.6839 & 0.6832 & 0.6832 & 0.7649 & 0.7637 & 0.3692 & 0.6839 & 0.5858 & 0.1996 & 0.2343 & 12.6 \\
            Kernel-porf\_softplus  & 0.6989 & 0.6989 & 0.6986 & 0.6986 & 0.7702 & 0.7700 & 0.3984 & 0.6989 & 0.5750 & 0.1961 & 0.2378 & 24.1 \\
            \bottomrule
        \end{tabular}%
    }
\end{table}
\vspace{-2mm}
Tables~\ref{tab:sst2_detailed_metrics}, \ref{tab:qqp_detailed_metrics}, and \ref{tab:rotten_tomatoes_detailed_metrics} show that kernel learning is feature-map dependent but often improves both discrimination and proper scoring metrics. On SST-2, Kernel-PORF-softplus gives the highest accuracy and MCC, while Kernel-FAVOR gives the lowest ECE. On QQP, Kernel-softmaxfeat is the strongest \(Q=K\) model and reduces LogLoss and Brier score relative to its vanilla counterpart. On Rotten Tomatoes, Kernel-softmaxfeat gives the best accuracy, MCC, LogLoss, and Brier score. The training-time column reflects the one-time alignment phase; after particles are learned, the serving-time attention formula remains linear in sequence length.

\paragraph{Additional experiment details.}
Inputs use the \texttt{bert-base-uncased} WordPiece tokenizer, are padded or truncated to 128 tokens, and use mean pooling over final token representations. The Transformer has two encoder layers, hidden size 128, two attention heads, feed-forward dimension 256, dropout 0.1, and 256 random features per head. Phase~A optimizes only the feature particles \(\bm{\Omega}_N\) using the alignment objective with Langevin noise, log-repulsion, norm projection, and early stopping; Phase~B freezes the learned particles and trains the remaining classifier with cross-entropy. Thus the extra cost appears during training, while inference uses the same linear-attention structure as the corresponding vanilla feature map.

\section{Conclusion}
We introduced a Coulomb/Riesz particle model for supervised random-feature learning and applied it to kernelized Transformer attention. The framework links label alignment, repulsive regularization, Langevin optimization, and equilibrium concentration in a single formulation. Empirically, the method improves several random-feature attention variants and provides a practical way to learn task-adaptive linear-attention kernels.

\paragraph{LLM usage statement.}
The authors used large language model (LLM) tools for writing assistance and code development. All LLM-assisted content was reviewed, verified, and edited by the authors, who take full responsibility for the correctness, originality, citations, proofs, experiments, figures, and final content of this paper.

\clearpage
\bibliography{bibliography}

\clearpage
\appendix
\section*{Supplementary Material}
\addcontentsline{toc}{section}{Supplementary Material}

\section{Related work}

To discuss related work, we first describe the kernel selection problem in the context of supervised learning problem. Consider a set of $n$ feature vectors and labels $\{(\bm{x}_i,y_i)\}_{i=1}^{n}, \bm{x}_{i}\in \mathcal{X},y_{i}\in \mathcal{Y}$.  We have a loss function $L:\real \times \real\rightarrow \real$, where $L(\cdot, y)$ is convex for $y\in \mathcal{Y}$, and a reproducing kernel Hilbert
space (RKHS) of functions $\mathcal{F}$ with kernel $K$. The $\ell_{2}$-regularized optimization problem that underlies the learning task of finding a function $f\in \mathcal{F}$ is as follows
\begin{align}
    \label{Eq: Primal-Dual}
\text{Primal:}\min_{f\in \mathcal{F}}\dfrac{1}{n}\sum_{i=1}^{n}L(f(\bm{x}_{i}),y_{i})+\dfrac{\lambda}{2}\|f\|_{\mathcal{H}}^{2}, \quad \text{Dual:} \max_{\boldsymbol{\alpha} \in \mathbb{R}^n} -\frac{1}{n} \sum_{i=1}^n L^*(\alpha_i,y_{i}) - \frac{1}{2\lambda} \boldsymbol{\alpha}^\top \mathbf{K} \boldsymbol{\alpha},
\end{align}
where $\|\cdot\|_{\mathcal{H}}$ is the Hilbert space norm, $\bm{\alpha}\in \mathbb{R}^{n}$ are dual variables,  $L^*(\alpha, y) = \sup_{z \in \mathbb{R}} \left\{ \alpha z - L(z, y) \right\}$ is the Fenchel conjugate of the loss function $L$, and $\bm{K}\stackrel{\text{def}}{=} [K_{ij}]\in \real^{n\times n}$ with $K_{ij}\stackrel{\text{def}}{=} K(\bm{x}_{i},\bm{x}_{j})$ is the kernel matrix.

While the kernel matrix $\mathbf{K}$ could, in principle, be optimized jointly with the dual variables, much of the literature instead focuses on approaches that decouple kernel learning from the estimation of $f \in \mathcal{F}$. A common strategy is to first construct or adapt the kernel---often by maximizing kernel--target alignment, which quantifies the similarity between the kernel and the target---before solving the regularized risk minimization problem (see, e.g., \cite{cortes2010two, cristianini2002kernel, kandola2002optimizing, lanckriet2004learning}). This alignment is formulated as the following optimization problem:
\begin{align} 
    \label{Eq: Kernel_Target_Alignment}
    \max_{K \in \mathcal{K}} \dfrac{1}{n(n-1)} \sum_{1\leq i\not= j \leq n}y_i y_j K(\bm{x}_i, \bm{x}_j), 
\end{align}
where $\mathcal{K}$ denotes a predefined class of kernel functions. Several approaches have been proposed in the literature for defining the class of kernels \(\mathcal{K}\) in kernel-based learning. In Table~\ref{tab:kernel_classes}, we provide a summary of common kernel class choices, along with their computational and memory complexities, as well as the corresponding references. Among the proposed approaches, \cite{sinha2016learning} stands out as a framework that seamlessly integrates with the random feature model, specifically by employing importance sampling of random features within the random feature-based kernel class. In contrast, we propose an alternative approach that yields improved performance. This enhancement is achieved by directly optimizing the distribution of random features within a particle optimization framework, as opposed to relying on importance sampling of random features.

\begin{table}[h]
    \centering
    \small
    \caption{Summary of kernel class $\mathcal{K}$ choices in the literature.     \textit{Notes.} \(n\): number of training samples; \(m\): number of base kernels; \(D\): number of random features; $d$: dimensionality of the input feature vectors; $L$: number of parameters or layers in the nonlinear transformation.}
    \label{tab:kernel_classes}
    \begin{tabular}{@{}p{2.2cm} p{6.8cm} p{1.6cm} p{1.6cm} p{1.5cm}@{}}
        \toprule
    \footnotesize \textbf{Kernel Class} & \footnotesize \textbf{Definition} & \footnotesize \textbf{Comp. Complexity} & \footnotesize \textbf{Memory Complexity} & \footnotesize \textbf{References} \\
        \midrule
        
        \textit{Convex combination of base kernels} & 
        \(\mathcal{K} = \left\{ K = \sum_{i=1}^{m} w_i K_i \,\middle|\, w_i \geq 0,\ \sum_{i=1}^{m} w_i = 1 \right\}\) \newline
        \(\bm{w} = (w_1, \dots, w_m) \in \mathbb{R}^m\) lies on the probability simplex. &
        \(\mathcal{O}(nm^2)\) & 
        \(\mathcal{O}(nm)\) &
        \cite{lanckriet2004learning, bach2004multiple} \\
        [0.2cm]

        \textit{Linear combination of base kernels} & 
        \(\mathcal{K} = \left\{ K = \sum_{i=1}^{m} w_i K_i \,\middle|\, \bm{w} \in \mathbb{R}^{m} \right\}\) \newline
        Allows negative weights; PSD constraints may be required. &
        \(\mathcal{O}(nm^2)\) & 
        \(\mathcal{O}(nm)\) &
        \cite{cortes2009learning} \\ [0.2cm]

        \textit{Nonlinear kernel combinations} & 
        \(\mathcal{K} = \left\{ K(\bm{x}, \bm{x}') = \sigma\left( \sum_{i=1}^{m} w_i K_i(\bm{x}, \bm{x}') + b \right) \right\}\) \newline
        \(\sigma\) is a nonlinearity (e.g., ReLU, sigmoid); \(w_i, b \in \mathbb{R}\). &
        \(\mathcal{O}(nmL)\) & 
        \(\mathcal{O}(nm)\) &
        \cite{wilson2016deep, ober2021promises} \\ [0.2cm]

        \textit{Parameterized kernels} & 
        \(\mathcal{K} = \left\{ K_\gamma(\bm{x}, \bm{x}') = \exp(-\gamma \|\bm{x} - \bm{x}'\|^2) \,\middle|\, \gamma \in \Gamma \right\}\) \newline
        \(\Gamma \subset \mathbb{R}_+\) is a bounded interval over which \(\gamma\) is optimized. &
        \(\mathcal{O}(n^2d)\) & 
        \(\mathcal{O}(n^2)\) &
        \cite{chapelle2002multiple} \\ [0.2cm]

        \textit{Random feature-based kernels} & 
        \(\mathcal{K} = \left\{ K_\mu(\bm{x}, \bm{x}') = \mathbb{E}_{\bm{\omega} \sim \mu}[\phi_{\bm{\omega}}(\bm{x}) \phi_{\bm{\omega}}(\bm{x}')] \,\middle|\, \mu \in \mathcal{M} \right\}\) \newline
        \(\phi_{\bm{\omega}}(\bm{x})\) is a random feature map; \(\mathcal{M}\) is a set of distributions over \(\bm{\omega}\). &
        \(\mathcal{O}(nD)\) & 
        \(\mathcal{O}(nD)\) &
        \cite{sinha2016learning} \\ [0.2cm]

        \textit{SDP-based kernel learning} & 
        \(\mathcal{K} = \left\{ \bm{K} \in \mathbb{S}_+^n \,\middle|\, \mathrm{tr}(\bm{K}) \leq c \right\}\) \newline
        \(\mathbb{S}_+^n\) denotes the set of \(n \times n\) PSD matrices; \(c > 0\) is a trace constraint. &
        \(\mathcal{O}(n^6)\) & 
        \(\mathcal{O}(n^2)\) &
        \cite{lanckriet2004learning} \\ [.2cm]

        \textit{Structured or hierarchical kernels} & 
        \(\mathcal{K} = \left\{ K = \sum_{i=1}^{m} w_i K_i \,\middle|\, \mathbf{w} \in \mathcal{W}_{\text{structured}} \right\}\) \newline
        \(\mathcal{W}_{\text{structured}}\) encodes priors such as group sparsity or tree-structured dependencies. &
        \(\mathcal{O}(nm \log m)\) & 
        \(\mathcal{O}(nm)\) &
        \cite{rakotomamonjy2008simplemkl, micchelli2005learning} \\

        \textit{Mean-field kernel approach} & 
        \(\mathcal{K} = \left\{ K_\mu(\bm{x}, \bm{x}') = \mathbb{E}_{\bm{\omega} \sim \mu}[\phi_{\bm{\omega}}(\bm{x}) \phi_{\bm{\omega}}(\bm{x}')] \,\middle|\, \mu \in \mathcal{M} \right\}\) \newline
        \(\phi_{\bm{\omega}}(\bm{x})\) is a random feature map; \(\mathcal{M}\) is a set of distributions over \(\bm{\omega}\). &
        \(\mathcal{O}(nD^{2})\)& 
        \(\mathcal{O}(nD)\) & 
                This work
        \\

        \bottomrule
    \end{tabular}
        \vspace{0.4em}
    \begin{minipage}{0.94\linewidth}
    \end{minipage}
\end{table}

\section{Proposed approach}

At a high level, we begin with a feature mapping to represent the kernel. Next, we learn a distribution that aligns this mapping with the labels using the kernel-target alignment (KTA) optimization formulated in Eq. \eqref{Eq: Kernel_Target_Alignment}. From this distribution, we sample random features, which are then used in a standard supervised learning framework.

Specifically, let $\phi: \mathcal{X}\times \Omega \rightarrow [-1,1]$, and $\mu$ denotes a probability measure on $\Omega$. We define the kernel
\begin{align}
    K_{\mu}(\bm{x}_{i},\bm{x}_{j})= \mathbb{E}_{\bm{\omega}\sim \mu}[\phi_{\bm{\omega}}(\bm{x}_{i})\phi_{\bm{\omega}}(\bm{x}_{j})],
\end{align}
where we used the shorthand notation $\phi_{\bm{\omega}}(\cdot)\stackrel{\text{def}}{=}\phi(\cdot;\bm{\omega})$. We optimize kernel $K_{\mu}$ over all distributions $\mu$ in some (large, nonparametric) set $\mathcal{M}$ of
possible distributions on random features

\begin{align}
    \label{Eq: Kernel-Target Align. w Random Features}
\sup_{\mu\in \mathcal{M}}\dfrac{1}{n(n-1)}\sum_{0\leq i\not = j\leq n}y_{i}y_{j}\mathbb{E}_{\bm{\omega}\sim \mu}[\phi_{\bm{\omega}}(\bm{x}_{i})\phi_{\bm{\omega}}(\bm{x}_{j})].
\end{align}

We consider independent, identically distributed (i.i.d.) samples or particles $\bm{\omega}_1, \dots, \bm{\omega}_N \stackrel{\text{i.i.d.}}{\sim} \mu$.\footnote{To distinguish between training samples and random feature samples, we refer to the latter as particles.}
Define the configuration of particles $\bm{\Omega}_{N} \defeq (\bm{\omega}_{1}, \dots, \bm{\omega}_{N})\in \Omega^{N}$, and the associated empirical distribution $\mu_{N}(\bm{\Omega}_{N}) = \frac{1}{N} \sum_{k=1}^{N} \delta_{\bm{\omega}_{k}}(\bm{\omega})$, where $\delta_{\bm{\omega}_{k}}(\cdot)$ is Dirac's delta function concentrated at $\bm{\omega}_{k}$. We consider the following regularized optimization problem to estimate the expectation term in Eq. \eqref{Eq: Kernel-Target Align. w Random Features} by optimizing the samples of the distribution, where the expectation is substituted by the Monte Carlo sample average approximation. In particular, we consider the following \textit{Gibbs point process} with  the Hamiltonian 
\begin{align}
    \label{Eq: Empirical Kernel-Target Align}
    \inf_{\mu \in \mathcal{M}_{N}} \mathcal{H}_{N}(\bm{\Omega}_{N})\defeq\mathcal{E}_{N}(\bm{\Omega}_{N}) + \lambda \mathcal{W}_{N,s}(\bm{\Omega}_{N}),
\end{align}
where $\lambda>0$ controls the strength of the interaction, the population loss function  is approximated as follows:
\begin{align}
\label{Eq:Empirical_loss}
    \mathcal{E}_{N}(\bm{\Omega}_{N}) \stackrel{\text{def}}{=} -\frac{1}{n(n-1)} \sum_{0\leq i\not = j\leq n} y_{i} y_{j} \frac{1}{N} \sum_{k=1}^{N} \phi_{\bm{\omega}_{k}}(\bm{x}_{i}) \phi_{\bm{\omega}_{k}}(\bm{x}_{j}).
\end{align}

Specifically, in the derivation of the empirical loss in Eq.~\eqref{Eq:Empirical_loss}, we replace the expectation under $\mu$ by integration with respect to the empirical measure $\mu_N \defeq \frac{1}{N}\sum_{k=1}^N \delta_{\bm{\omega}_k}$:
\begin{align}
    \label{Eq:Monte_Carlo}
    K_{\mu}(\bm{x}_{i},\bm{x}_{j})
    &=\int_{\Omega} \phi_{\bm{\omega}}(\bm{x}_{i})\,\phi_{\bm{\omega}}(\bm{x}_{j})\,\mu(d\bm{\omega}) \nonumber\\
    &\approx \int_{\Omega} \phi_{\bm{\omega}}(\bm{x}_{i})\,\phi_{\bm{\omega}}(\bm{x}_{j})\,\mu_N(d\bm{\omega})
    =K_{\mu_N}(\bm{x}_{i},\bm{x}_{j}) \nonumber\\
    &=\frac{1}{N} \sum_{k=1}^{N} \phi_{\bm{\omega}_{k}}(\bm{x}_{i}) \, \phi_{\bm{\omega}_{k}}(\bm{x}_{j}).
\end{align}

The following lemma provides the explicit form of this energy function for the random Fourier feature models in \cite{rahimi2007random}:

\begin{lemma}[Empirical KTA equals trigonometric energy with random bias]
    \label{lemma:KTA_energy_bias}
    Let $\{(\bm{x}_i,y_i)\}_{i=1}^n$ with $\bm{x}_i\in\mathbb{R}^d$, $y_i\in\{-1,+1\}$. 
    Let $\bm{\Omega}_{N}\defeq(\bm{\omega}_1,\ldots,\bm{\omega}_N)\in\Omega^N$ be the particle configuration and 
    $\bm{b}\defeq(b_1,\ldots,b_N)^\top$ with $b_k\stackrel{\text{i.i.d.}}{\sim}\mathrm{Unif}[0,2\pi]$. 
    Define $\phi_{\bm{\omega},b}(\bm{x})=\sqrt{2}\cos(\bm{\omega}^\top\bm{x}+b)$. 
    Consider
    \begin{align}
        \mathcal{E}_N(\bm{\Omega}_N)
        \stackrel{\mathrm{def}}{=}
        -\frac{1}{n(n-1)}\sum_{0\le i\neq j\le n}y_i y_j\;\frac{1}{N}\sum_{k=1}^N
        \phi_{\bm{\omega}_k,b_k}(\bm{x}_i)\,\phi_{\bm{\omega}_k,b_k}(\bm{x}_j).
    \end{align}
    Let $\bm{X}\in\mathbb{R}^{n\times d}$ stack the $\bm{x}_i^\top$, and write
    \[
    \bm{A}\;=\;\bm{X}\bm{\Omega}_N^{\!\top}+\bm{1}_n\bm{b}^{\top},\qquad
    \bm{C}=\cos(\bm{A})\in\mathbb{R}^{n\times N},\quad
    \bm{S}=\sin(\bm{A})\in\mathbb{R}^{n\times N},
    \]
    where $\bm{1}_n$ is the all-ones vector, and $\cos(\cdot)$ and $\sin(\cdot)$ are applied to each element of the matrix.
    Then, up to an additive constant independent of $(\bm{\Omega}_N,\bm{b})$,
    \begin{align}
        \mathcal{E}_N(\bm{\Omega}_N)
        \;\equiv\;
        -\frac{2}{n^2}\,\big\|\bm{y}^\top \bm{C}\big\|_2^2,
    \end{align}
    and, taking expectation over the random phases $\bm{b}$,
    \begin{align}
        \label{Eq:FourierFeature_EnergyFunction}
        \mathbb{E}_{\bm{b}}\!\big[\mathcal{E}_N(\bm{\Omega}_N)\big]
        &\;\equiv\;
        -\frac{1}{Nn(n-1)}\left(
        \big\|\bm{y}^\top \cos(\bm{X}\bm{\Omega}_N^{\!\top})\big\|_2^2
        +
        \big\|\bm{y}^\top \sin(\bm{X}\bm{\Omega}_N^{\!\top})\big\|_2^2
        \right)\\
        &=-\frac{1}{N\,n(n-1)}
        \sum_{k=1}^{N}
        \left|
        \sum_{i=1}^{n}
        y_i \, e^{\,i\,\boldsymbol{\omega}_k^{\!\top}\boldsymbol{x}_i}
        \right|^2.
    \end{align}
\end{lemma}

Furthermore, the regularization term in Eq. \eqref{Eq: Empirical Kernel-Target Align} captures the interaction energy of every sample $\bm{\Omega}_{k}$ with all the other samples $\bm{\Omega}_{\ell}$ in an infinite configuration 
\begin{align}
    \mathcal{W}_{N,s}(\bm{\Omega}_{N})=\frac{1}{2N(N-1)}\sum_{1\leq k\not =\ell\leq N}g_{s}(\bm{\omega}_{k}-\bm{\omega}_{\ell}),
\end{align}
where $g_{s}$ is the homogeneous potential of degree $s$,
\[
\label{Eq:homogeneous_potential}
g_s(\bm{\omega}) \defeq
\begin{cases}
    \| \bm{\omega}\|_{2}^{-s}, & \text{for } s \in (0, +\infty], \\
    -\log \|\bm{\omega}\|_{2}, & \text{for } s = 0, \\
    -\|\bm{\omega}\|_{2}^{-s}, & \text{for } s \in (-2, 0).
\end{cases}
\]

\begin{algorithm2e}[t]
\caption{Riesz/Coulomb Particles for Kernel Estimation in Kernel Methods}
\label{alg:langevin_kernel}
\SetAlgoLined
\KwData{number of particles \(N\); initial particle positions \(\{\bm{\omega}_k^0\}_{k=1}^N\); step size \(\eta\); inverse temperature \(\beta\); threshold \(\delta\); divergence \(\mathcal{D}(\cdot,\cdot)\); number of random feature samples \(D<N\)}
\KwResult{estimated kernel \(\widehat{K}(\bm{x},\bm{x}')\)}
\(\mu_N^0 \gets N^{-1}\sum_{k=1}^N \delta_{\bm{\omega}_k^0}\)\;
\(m \gets 0\)\;
\Repeat{\(\mathcal{D}(\mu_N^{m},\mu_N^{m-1}) < \delta\)}{
  \For{\(k=1\) \KwTo \(N\)}{
    sample \(\bm{\xi}_k^m \sim \mathcal{N}(\bm{0},\bm{I}_d)\)\;
    \(\nabla\mathcal{H}_N \gets N\nabla_{\bm{\omega}_k^m}\mathcal{H}_N(\bm{\Omega}_N^m)\)\;
    \(\bm{\omega}_k^{m+1} \gets \mathcal{P}_\Omega\!\left(\bm{\omega}_k^m-\eta\nabla\mathcal{H}_N+\sqrt{2\eta/\beta}\,\bm{\xi}_k^m\right)\)\;
  }
  \(\mu_N^{m+1} \gets N^{-1}\sum_{k=1}^N \delta_{\bm{\omega}_k^{m+1}}\)\;
  \(m \gets m+1\)\;
}
Compute weights \(w_k \gets \exp[-\beta\mathcal{H}(\bm{\omega}_k^m)]/\sum_{j=1}^{N}\exp[-\beta\mathcal{H}(\bm{\omega}_j^m)]\)\;
Sample \(D\) particles \(\{\bm{\omega}_k^\ast\}_{k=1}^{D}\) according to \(\{w_k\}_{k=1}^{N}\)\;
\(\widehat{K}(\bm{x},\bm{x}') \gets D^{-1}\sum_{k=1}^{D}\phi_{\bm{\omega}_k^\ast}(\bm{x})\phi_{\bm{\omega}_k^\ast}(\bm{x}')\)\;
\Return{\(\widehat{K}(\bm{x},\bm{x}')\)}\;
\end{algorithm2e}

Note that the Hamiltonian in Eq.~\eqref{Eq: Empirical Kernel-Target Align} uses the standard mean-field normalization of the pairwise interaction. Under this normalization, the interaction energy is an average over particle pairs, and its large-\(N\) limit is the corresponding continuum Riesz energy. Therefore, we do not introduce an additional factor \(N^{s/d}\) in either the loss term or the interaction term. The heuristic spacing \(N^{-1/d}\) describes nearest-neighbor distances among \(N\) well-distributed particles in \(d\) dimensions, whereas the distance between a typical pair of independently sampled particles from a fixed limiting distribution remains \(O(1)\). Thus, multiplying the Hamiltonian by \(N^{s/d}\) would correspond to a different, \(N\)-dependent choice of interaction strength rather than to the mean-field scaling considered here. We regard the coordinates in \(\Omega\) as nondimensionalized; equivalently, any fixed length-scale normalization, such as replacing \(g_s(\bm{\omega})\) by \(g_s(\bm{\omega}/\ell_0)\), can be absorbed into the coupling parameter \(\lambda\) for \(s>0\), while in the logarithmic case it changes the energy only by an additive constant and hence leaves the induced dynamics unchanged. Viewing the particles as \textit{charged} samples in \(\Omega\), the interaction in Eq.~\eqref{Eq:homogeneous_potential} corresponds to the \textit{Riesz potential} for general \(s>0\), with the Coulomb/Newtonian case obtained at \(s=d-2\) for \(d\geq 3\) and the logarithmic Coulomb kernel at \(s=0\) in \(d=2\). The sign convention is chosen so that the induced pairwise force is repulsive.

These repulsive interactions act as a regularization mechanism, preventing the particles from collapsing into a single point mass (Dirac delta) and promoting their dispersion across the domain $\Omega$. In particular, they encourage the support of the empirical distribution $\widehat{\mu}_N$ to approximate the support of the true underlying distribution $\Omega = \mathrm{supp}(\mu)$, where $\mathrm{supp}(\cdot)$ denotes the support of a distribution. 

From a statistical learning theory perspective, such repulsive forces mitigate sample clustering, thereby enhancing the expressiveness of the kernel. By maintaining spatial diversity among the samples, the kernel can better capture the underlying structure of the data. Numerical simulations confirm that this regularization indeed improves both the stability of the particle system and the generalization performance of kernel-based models.

\subsection{Langevin dynamics for efficiently solving Eq. \eqref{Eq: Empirical Kernel-Target Align}}

We optimize the positions of the samples in Eq. \eqref{Eq: Empirical Kernel-Target Align} using Langevin dynamics. From an optimization perspective, we model the evolution of the particle system over time steps $m = 0, 1, \dots, T-1$ according to the following Langevin update rule:
\begin{align}
\bm{\omega}_{k}^{m+1}=\mathcal{P}_{\Omega}\left(\bm{\omega}_{k}^{m}-\eta\,N \nabla_{\bm{\omega}_{k}^{m}}\mathcal{H}_{N}(\bm{\Omega}_{N}^{m})+\sqrt{\frac{2\eta}{\beta_{N}}}\cdot\bm{\xi}^{m}_{k}\right), \quad k=1,2,\cdots,N,
\end{align}
where $\eta \defeq \eta(n,N)$ is the step size that scales with both the number of training samples $n$ and particles $N$, $\beta_{N}$ is the inverse temperature that depends on the number of particles,  $\bm{\xi}^{m}_{k}\stackrel{\text{i.i.d.}}{\sim} \mathsf{N}(\bm{0},\bm{I}_{d\times d})$ denotes i.i.d. Gaussian noise with zero mean and identity covariance matrix.  Moreover, $\mathcal{P}_{\Omega}(\cdot)$ is the Euclidean projection onto the set $\Omega$.  The particles are initially sampled independently from a probability distribution \( \mu_0 \), i.e., \( \bm{\omega}_{1}^{0}, \dots, \bm{\omega}_{N}^{0} \overset{\text{i.i.d.}}{\sim} \mu_0 \).

This stochastic update rule blends deterministic gradient descent (on the energy landscape defined by the Hamiltonian $\mathcal{H}_{N}$) with random perturbations, allowing the system to approximate samples from a Gibbs distribution 
under appropriate conditions. The projection step ensures the dynamics are constrained to a feasible domain, which may encode structural or regularization constraints critical to the optimization problem. In Algorithm \ref{alg:langevin_kernel}, we summarize these steps in a kernel learning algorithm.

Using the random feature samples, we construct the randomized feature map 
\begin{align}
\bm{\phi}_{D}(\bm{x}) = \left( \phi_{\bm{\omega}^{\ast}_{1}}(\bm{x}), \ldots, \phi_{\bm{\omega}^{\ast}_{D}}(\bm{x}) \right), 
\end{align}
and define the corresponding RKHS-based function class as
\[
\mathcal{F} = \left\{ f : f(\bm{x}; \bm{\theta}) = \frac{1}{\sqrt{D}} \langle \bm{\theta}, \bm{\phi}_{D}(\bm{x}) \rangle, \; \bm{\theta} \in \Theta \subset \mathbb{R}^D \right\}.
\]
Under this construction, the primal objective in Eq.~\eqref{Eq: Primal-Dual} transforms into the following finite-dimensional optimization problem:
\[
\max_{\bm{\theta} \in \Theta} \sum_{i=1}^{n} L\left( \frac{1}{\sqrt{D}} \bm{\theta}^T \bm{\phi}_{D}(\bm{x}_i), y_i \right) + \frac{\lambda}{2D} \| \bm{\theta} \|_2^2.
\]

\section{Kernelized attention in transformer architecture}

Transformers hinge on the self-attention operation \citep{vaswani2017attention}. Let the (embedded) input sequence be
\(\bm{X}\defeq(\bm{x}_{0},\bm{x}_1,\ldots,\bm{x}_{\ell-1})\in\real^{\ell\times d_x}\), where $\bm{x}_{0}\defeq \bm{x}_{\mathrm{CLS}}$ is the $\mathrm{CLS}$ token. Queries, keys, and values are produced by linear maps
\[
\bm{Q}(\bm{X})=\bm{X}\bm{W}_Q,\qquad
\bm{K}(\bm{X})=\bm{X}\bm{W}_K,\qquad
\bm{V}(\bm{X})=\bm{X}\bm{W}_V,
\]
with \(\bm{W}_Q,\bm{W}_K\in\real^{d_x\times d}\) and \(\bm{W}_V\in\real^{d_x\times d_v}\). Standard (scaled dot-product) attention is
\[
\bm{A}(\bm{X})
=\operatorname{softmax}\!\Big(\tfrac{1}{\sqrt{d}}\,
\bm{Q}(\bm{X})\bm{K}(\bm{X})^{\top}\Big)\,\bm{V}(\bm{X}).
\]

Elementwise, $\bm{A}(\bm{X})=(\bm{a}_{i}(\bm{X}))_{i=0}^{\ell}$, where the \(i\)-th output \(\bm{a}_i\in\real^{d_v}\) is a
\emph{normalized kernel smoother} (\textit{a.k.a.} Nadaraya--Watson estimator)
\begin{equation}
    \label{eq:kernel-smoother}
    \bm{a}_{i}(\bm{X}) 
    =\frac{\sum_{j=0}^{\ell}K(\bm{q}_i,\bm{k}_j)\,\bm{v}_j}
    {\sum_{j=0}^{\ell}K(\bm{q}_i,\bm{k}_j)},
    \qquad
    K(\bm{q},\bm{k})=\exp\!\Big(\tfrac{1}{\sqrt{d}}\bm{q}^{\top}\bm{k}\Big),
\end{equation}
which connects attention to classical nonparametric regression
\citep{nadaraya1964,watson1964} and to kernel interpretations of Transformers
\citep{tsai2019kernelview}. Thus, softmax attention is attention with a specific positive kernel
\(K\); the denominator enforces a convex combination and stabilizes the estimator.

Replacing softmax with any positive kernel \(K\) yields a family of attentions with controllable
inductive bias (locality, smoothness, anisotropy). When \(K\) is approximated via a random feature map
\(\bm{\phi}_{\bm{\Omega}_{N}}:\real^d\to\real^N,\ \bm{x}\mapsto \bm{\phi}_{\bm{\Omega}_{N}}\defeq N^{-1/2}(\phi_{\bm{\omega}_i}(\bm{x}))_{i=1}^N\) such that \(K_{{\bm{\Omega}_{N}}}(\bm{q},\bm{k})=
\bm{\phi}_{{\bm{\Omega}_{N}}}(\bm{q})^{\top}\bm{\phi}_{{\bm{\Omega}_{N}}}(\bm{k})\), we obtain a \emph{linearized} form
\citep{katharopoulos2020linear,choromanski2021performer,peng2021rfa}:
\begin{align}
    \label{eq:linear-attn}
    \widehat{\bm{a}}_{i}(\bm{X})
    &=\frac{\bm{\phi}_{\bm{\Omega}_{N}}(\bm{q}_i)^{\top}\!\Big(\sum_{j=1}^{n}\bm{\phi}_{\bm{\Omega}_{N}}(\bm{k}_j)\,\bm{v}_j^{\top}\Big)}
    {\bm{\phi}_{\bm{\Omega}_{N}}(\bm{q}_i)^{\top}\!\Big(\sum_{j=1}^{n}\bm{\phi}_{\bm{\Omega}_{N}}(\bm{k}_j)\Big)}
    =\frac{\bm{\phi}_{\bm{\Omega}_{N}}(\bm{q}_i)^{\top} \bm{G}}{\bm{\phi}_{\bm{\Omega}_{N}}(\bm{q}_i)^{\top}\bm{z}},
    \\
    \bm{G}&\defeq\sum_{j=0}^{\ell}\bm{\phi}_{\bm{\Omega}_{N}}(\bm{k}_j)\,\bm{v}_j^{\top}\in\real^{N\times d_v},\qquad
    \bm{z}\defeq\sum_{j=0}^{\ell}\bm{\phi}_{\bm{\Omega}_{N}}(\bm{k}_j)\in\real^{N}.
\end{align}
This reduces complexity from \(\mathcal{O}(\ell^2 d)\) to \(\mathcal{O}(\ell N + N d_v)\) per head, with memory \(\mathcal{O}(N d_v)\),
while preserving the normalized kernel-smoothing structure; causal/padding masks apply by omitting
the masked terms in the sums. Alternative sub-quadratic routes include Nystr{\"o}m approximations
and kernelized attention with relative positional encoding
\citep{xiong2021nystromformer,chen2021skyformer,luo2021kernelrpe}.

\subsection{Alignment for attention kernel.}
Since labels depend on the entire input sequence through
\(\big(\bm{Q}(\bm{X}),\bm{K}(\bm{X}),\bm{V}(\bm{X})\big)\),
the alignment problem departs from the usual random feature setup in Eq. \eqref{Eq: Kernel-Target Align. w Random Features}. For a one-layer Transformer, collect the attention outputs into
\[
\bm{\Phi}(\bm{X})\defeq\big(\widehat{\bm{a}}_0(\bm{X}),\ldots,\widehat{\bm{a}}_\ell(\bm{X})\big)\in\real^{\ell\times d_v}.
\]
A linear classifier can then be applied either

\begin{itemize}
    \item at token level for token classification problem (e.g., NER): \(\hat{y}_i = \bm{w}^\top \bm{a}_i(\bm{X}) + b\), \(i=1,2,\ldots,\ell \),
    
    \item at sequence level via a pooling operator \(P:\real^{\ell \times d_v}\to\real^{d_v}\):
    \[
    \hat{y}= \bm{w}^\top P\!\big(\bm{\Phi}(\bm{X})\big)+b,
    \qquad
    P\!\big(\bm{\Phi}(\bm{X})\big)=\sum_{i=0}^\ell \pi_i\,\bm{a}_i(\bm{X}),\ \ \bm{\pi}\in\Delta^{\ell}.
    \]
    Here \(\Delta^{\ell}\) is the probability simplex, ensuring a convex combination. Examples:
    \([\mathrm{CLS}]\): \(P(\bm{\Phi})=\widehat{\bm{a}}_{0}(\bm{X})\) (i.e., \(\pi_{0}=1\), others \(\pi_{k}=0, \forall 0<k\leq \ell \)) where \(\bm{X}\defeq (\bm{x}_{0},\bm{x}_{1},\ldots,\bm{x}_{\ell})\);
    mean pooling: \(\pi_i=\tfrac{1}{\ell}\) for all \(i=1,\cdots,\ell\) and $\pi_{0}=0$.
    (With padding/masks \(m_i\in\{0,1\}\), use \(\pi_i=\tfrac{m_i}{\sum_{j=1}^{n} m_j}\) for all $i=1,\cdots,\ell$.)
\end{itemize}
This mirros the classical random-feature setting---where a random map $\bm{\phi}(\bm{x})$ feeds a linear classifier---but here the attention features $\bm{\Phi}(\bm{X})$ (and their pooled variant) are functions of the entire sequence $\bm{X}$. For sequence-level multi-class classification, given training sequences and labels $\{(\bm{X}_i, y_i)\}_{i=1}^n$ where the labels are $y_{i}\in \mathcal{Y}=\{1,2,\cdots,m\}$, and fixed embedding matrices $(\bm{W}_Q,\bm{W}_K,\bm{W}_V)$, we optimize the target-alignment objective
\begin{align}
    \max_{\mu\in\mathcal{M}_N}\ \mathcal{V}^{\text{seq}}_N(\bm{\Omega}_N)
    \;
=\frac{1}{n(n-1)}\sum_{1\le i<j\le n}
\big(\bm{e}_{y_i}^{\!\top}\bm{e}_{y_j}\big)\,
P({\bm{\Phi}}(\bm{X}_i))^{\!\top}P(\bm{\Phi}(\bm{X}_j)),
\end{align}
where $P(\bm{\Phi}(\bm{X}_i))\in\mathbb{R}^{d_v}$ is the pooled representation of sequence $\bm{X}_i$, and  $\{\bm{e}_k\}_{k=1}^m$ denotes the standard basis of $\mathbb{R}^m$ which hot-encode the label. Note that the dependence on the particle set $\bm{\Omega}_N$ (equivalently, on $\mu$) enters only through the sequence embeddings $\bm{\Phi}(\cdot)$ induced by the attention features and the pooling operator $P$. In particular,
 $\bm{\Omega}_{N}=(\bm{\omega}_{1},\cdots,\bm{\omega}_{N})$ is implicit in the definition of attention features $\bm{\Phi}(\bm{X})$ through each coordinate $\widehat{\bm{a}}_{i}(\bm{X}),i=0,\cdots,\ell$. Similarly, for token level classification, consider the training dataset $\{(\bm{x}_{i,1},y_{i,1}),\cdots,(\bm{x}_{i,\ell},y_{i,\ell}))\}_{i=1}^{n}$. Then, the target alignment problem reads
\begin{align}
    \max_{\mu\in \mathcal{M}_{N}} \ \mathcal{V}^{\text{tok}}_{N}(\bm{\Omega}_{N})=\frac{1}{|\mathcal{U}|(|\mathcal{U}|-1)}
    \sum_{\substack{(i,t)<(j,s)\\ (i,t),(j,s)\in\mathcal{U}}}
    \big(\bm{e}_{y_{i,t}}^{\!\top}\bm{e}_{y_{j,s}}\big)\,
    \widehat{\bm{a}}_{i,t}(\bm{X}_i)^{\!\top}\widehat{\bm{a}}_{j,s}(\bm{X}_j).
\end{align}
We minimize the \emph{energy} $\mathcal{V}^{\text{seq}}_{N}(\bm{\Omega}_{N})$ and $\mathcal{V}^{\text{tok}}_{N}(\bm{\Omega}_{N})$ in conjunction with the Coloumb/Riesz regularizer, for sequence level and token level classification, respectively.

\textit{Remark.}
In standard Transformer blocks, the attention sublayer is followed by a position-wise feed-forward network (FFN), with each sublayer wrapped by residual connections and layer normalization. A canonical FFN acts independently at each position:
\[
\mathrm{FFN}(\bm{a}_i)=\bm{W}_2\,\sigma(\bm{W}_1 \bm{a}_i(\bm{X}) + \bm{b}_1)+\bm{b}_2,\quad
\bm{W}_1\!\in\!\real^{d_v\times d_{\mathrm{ff}}},\ \bm{W}_2\!\in\!\real^{d_{\mathrm{ff}}\times d_v},
\]
with nonlinearity \(\sigma\) (e.g., GELU). Using \(\mathrm{LN}\) for LayerNorm, the residual/normalization update is, schematically,
 \(\tilde{\bm{a}}_i=\mathrm{LayerNorm}\!\big(\bm{a}_i+\mathrm{FFN}(\bm{a}_i)\big)\).
Because the FFN is parameter-shared across positions and applied pointwise, these operations (i) preserve sequence length and token indices; (ii) leave the attention weights and the kernel-smoothing form that produced \(\bm{a}_i(\bm{X})\) unchanged; and (iii) implement a learned, per-token reparameterization of attention outputs that improves expressivity and optimization stability. We include FFN, residual, and normalization components in our numerical experiments; the linear classifier discussed above is used to isolate the representation induced by attention and to draw a precise parallel with classical random-feature models.

\providecommand{\defeq}{\vcentcolon=}
\providecommand{\relu}{\operatorname{ReLU}}

\section{Positive random-feature maps for linearized attention}

We use linearized attention in normalized Nadaraya--Watson form (cf.\ \citealp{pmlr-v119-katharopoulos20a,vaswani2017attention}), instantiated by \emph{positive} feature maps \(\bm{\phi}:\mathbb{R}^{d}\!\to\!\mathbb{R}^{M}_{+}\) so that the kernel
\(K(\bm{q},\bm{k})\defeq \bm{\phi}(\bm{q})^{\!\top}\bm{\phi}(\bm{k})\) is positive semidefinite and the normalization is well-posed. Throughout, we use a \emph{particle} parameterization with columns \(\bm{\omega}_1,\dots,\bm{\omega}_M\) of \(\bm{\Omega}\!\in\!\mathbb{R}^{d\times M}\) learned from data (as opposed to fixed i.i.d.\ draws \citealp{choromanski2020performer,peng2021rfa}). Stabilizations used in practice mirror the code: a temperature \(\tau>0\), a small positive floor, and per-token \(\ell_1\) re-normalization when stated.

\paragraph{Positive exponential random features (\textsc{favor}).}
The exponential map
\[
\bm{\phi}^{\textsc{favor}}_{\bm{\Omega}}(\bm{x})
\;\defeq\;
\frac{1}{\sqrt{M}}\Big(\exp(\bm{\omega}_i^{\!\top}\bm{x}-\tfrac12\|\bm{x}\|_2^2)\Big)_{i=1}^{M}
\]
yields
\(K(\bm{q},\bm{k})=\tfrac1M\sum_{i=1}^{M}\exp(\bm{\omega}_i^{\!\top}\bm{q}+\bm{\omega}_i^{\!\top}\bm{k}-\tfrac12\|\bm{q}\|^2-\tfrac12\|\bm{k}\|^2)\ge 0\).
With \(\bm{\omega}_i\!\sim\!\mathcal{N}(0,\mathbf{I})\) and \(M\!\to\!\infty\) this recovers a Monte Carlo approximation to the softmax kernel \citep{choromanski2020performer,peng2021rfa}; we instead \emph{learn} \(\bm{\Omega}\) to obtain a task-adaptive kernel.

\paragraph{Deterministic positive features (ELU\(+1\) with floor + \(\ell_1\) renorm).}
Following \cite{pmlr-v119-katharopoulos20a}, we use
\[
\bm{\phi}^{\mathrm{elu+1}}_{\bm{\Omega}}(\bm{x})
\;\defeq\;
\frac{1}{Z(\bm{x})}\Big(1+\mathrm{ELU}\!\big((\bm{\omega}_i^{\!\top}\bm{x})/\tau + b_i\big)\Big)_{i=1}^{M},
\quad
Z(\bm{x}) \propto \sum_{i} \max\!\{ \mathrm{ELU}(\cdot)+1,\, \text{floor} \},
\]
where $\mathrm{ELU}$ is from \cite{clevert2015fast}. We clamp to a small floor and re-normalize per token so that \(\sum_i \phi_i(\bm{x})=\sqrt{M}\).

\paragraph{Softplus features (floor + \(\ell_1\) renorm).}
\[
\bm{\phi}^{\mathrm{softplus}}_{\bm{\Omega}}(\bm{x})
\;\defeq\;
\frac{1}{Z(\bm{x})}\Big(\mathrm{softplus}\!\big((\bm{\omega}_i^{\!\top}\bm{x})/\tau\big)+\text{floor}\Big)_{i=1}^{M},
\]
again strictly positive, floor-stabilized, and re-normalized; see \cite{dugas2001incorporating} for softplus.

\paragraph{Squared-sigmoid features (\(\ell_1\) renorm).}
\[
\bm{\phi}^{\mathrm{sigmoid2}}_{\bm{\Omega}}(\bm{x})
\;\defeq\;
\frac{1}{Z(\bm{x})}\Big(\sigma\!\big((\bm{\omega}_i^{\!\top}\bm{x})/\tau\big)^2\Big)_{i=1}^{M},
\quad \sigma(t)=\tfrac{1}{1+e^{-t}},
\]
which are non-negative and re-normalized per token.

\paragraph{Softmax-over-features (per-token).}
We also consider a \emph{feature-softmax} map
\[
\bm{\phi}^{\mathrm{softmaxfeat}}_{\bm{\Omega}}(\bm{x})
\;\defeq\;
\sqrt{M}\,\mathrm{softmax}\!\big((\bm{\Omega}^{\!\top}\bm{x})/\tau\big),
\]
which is strictly positive and sums to \(\sqrt{M}\) by construction.

\paragraph{Cosine-squared random features.}
Motivated by random Fourier features for shift-invariant kernels \citep{rahimi2007random,rahimi2009random}, we use
\[
\bm{\phi}^{\mathrm{cos2}}_{\bm{\Omega},\bm{b}}(\bm{x})
\;\defeq\;
\frac{1}{Z(\bm{x})}\Big(\cos(\bm{\omega}_i^{\!\top}\bm{x}+b_i)^2\Big)_{i=1}^{M},
\]
with fixed phases \(\bm{b}\); non-negativity is immediate, and we apply per-token \(\ell_1\) renormalization.

\paragraph{PORF-Softplus (orthogonal initialization).}
To reduce variance and improve conditioning, we initialize \(\bm{\Omega}\) with \emph{orthogonal random features} blocks \citep{yu2016orthogonal,choromanski2017sorf} and then apply the softplus map above:
\[
\bm{\phi}^{\mathrm{porf\mbox{-}softplus}}_{\bm{\Omega}}(\bm{x})
\;\equiv\;
\bm{\phi}^{\mathrm{softplus}}_{\bm{\Omega}}(\bm{x}),\quad
\bm{\Omega}^\top \bm{\Omega} \approx \mathbf{I}.
\]
The parameters remain learnable after orthogonal initialization.

\paragraph{Learning the kernel via alignment/KTA (Phase A).}
Beyond a vanilla cross-entropy training of the classifier head (Phase B), we \emph{first} adapt \(\bm{\Omega}\) with a representation-level objective: either a within-class alignment loss (maximizing same-class similarity) or \emph{Kernel Target Alignment} (KTA; \citealp{cristianini2002kernel}) using centered Gram matrices \(K\) and label kernel \(Y\). This yields task-adaptive positive kernels while preserving linear-time forward/backward passes. We evaluate on SST-2 from GLUE \citep{wang2018glue} with BERT tokenization \citep{devlin2019bert}, matching the experimental setup in our code.

\medskip
\noindent
\textbf{Implementation details.}
All maps use a temperature \(\tau\), small positive floors where applicable, and (except \textsc{favor} and feature-softmax) per-token \(\ell_1\) re-normalization to keep \(\sum_i \phi_i(\bm{x})=\sqrt{M}\).
Columns of \(\bm{\Omega}\) are optionally constrained by column-wise \(\ell_2\) clipping during Phase A to stabilize learning.

\section{Theoretical results}

Before presenting our theoretical results, we first outline the key assumptions that underpin our analysis:

\begin{assumption}[Initial distribution of particles]
 The particles are initially sampled independently from a probability distribution \( \mu_0 \) that admits a Lebesgue density \( \rho_0(\bm{\omega})=\dfrac{\mu_{0}(\mathrm{d}\bm{\Omega})}{\mathrm{d}\bm{\omega}} \).
\end{assumption}

\begin{assumption}[Constant Temperature]
    The temperature parameter $\beta_{N}$ remains constant and finite throughout the dynamics for a fixed number of particles $N$.
\end{assumption}

\begin{assumption}[Projection space]
The feature domain $\Omega\subset\real^{d}$ is compact and convex with $C^2$ boundary. The projected dynamics use non-absorbing, reflective boundary conditions on $\partial\Omega$.
\end{assumption}

\begin{assumption}[Bounded Random Feature Embedding]
    Let $\Phi(\bm{x}) \in L^2(\Omega,\mu_0)$ denote the 
    random feature embedding associated with the kernel $K(\bm{x},\bm{x}')$. We assume that
    \begin{align}
        \sup_{x \in \mathcal X} \|\Phi(x)\|_{L^2(\Omega,\mu_0)} 
        = \sup_{x \in \mathcal X} \sqrt{K(x,x)} 
        \le L_\Phi < \infty.
    \end{align}
\end{assumption}

\subsection{Continuity equations and equilibrium distribution}
\begin{theorem}[Projected-particle mean-field continuity equation]
    \label{Theorem: Continuity Equation}
    Let \(\Omega\subset\real^d\) be compact and convex with \(C^2\) boundary and outward unit normal \(\bm n\). Assume that \(V\in C^2(\overline\Omega)\), that \(g_s\) is smoothly regularized or the particle system remains collision-free, and that \(\mu_N^0\Rightarrow\rho_0(\bm\omega)d\bm\omega\). Consider the mean-field-scaled projected Langevin iteration
    \begin{align}
    \label{eq:supp-projected-mf-euler}
    \bm\omega_k^{m+1}
    =\mathcal P_{\overline\Omega}\!\left(
    \bm\omega_k^m-
    \eta_N\left[\nabla V(\bm\omega_k^m)+\frac{\lambda}{N-1}\sum_{\ell\ne k}\nabla g_s(\bm\omega_k^m-\bm\omega_\ell^m)\right]
    +\sqrt{2\eta_N/\beta}\,\bm\xi_k^m
    \right),
    \end{align}
    where \(\bm\xi_k^m\stackrel{\mathrm{i.i.d.}}{\sim}\mathsf N(\bm{0},I_d)\), \(\eta_N\downarrow0\), and the projected-Euler consistency error vanishes as \(N\to\infty\). Let \(\mu_t^N=N^{-1}\sum_{k=1}^N\delta_{\bm\omega_k^{\lfloor t/\eta_N\rfloor}}\). Then \(\mu_t^N\Rightarrow\mu_t=\rho_t(\bm\omega)d\bm\omega\) in probability, uniformly for \(t\) in compact intervals. The limiting density is governed by the McKean--Vlasov equation
    \begin{align}
    \label{eq:fokker-planck}
    \frac{\partial \rho_{t}(\bm{\omega})}{\partial t}
    = \nabla_{\bm{\omega}} \cdot \left[ \rho_{t}(\bm{\omega}) \nabla_{\bm{\omega}} U_t(\bm\omega) \right]
    + \frac{1}{\beta} \Delta_{\bm{\omega}} \rho_{t}(\bm{\omega}),
    \qquad (t,\bm\omega)\in(0,T]\times\Omega,
    \end{align}
    where
    \begin{align}
    U_t(\bm\omega)=V(\bm\omega)+\lambda\int_\Omega g_s(\bm\omega-\bm\omega')\rho_t(\bm\omega')\,d\bm\omega'.
    \end{align}
    The PDE is supplemented by the initial datum, the Robin/no-flux boundary condition, and conservation of probability:
    \begin{align}
    \label{eq:fokker-planck-boundary}
    \rho_t|_{t=0}=\rho_0,
    \qquad
    \left(\rho_t\nabla_{\bm\omega}U_t+\beta^{-1}\nabla_{\bm\omega}\rho_t\right)\cdot\bm n=0
    \quad\text{on }(0,T]\times\partial\Omega,
    \qquad
    \rho_t\ge0,
    \quad
    \int_\Omega\rho_t(\bm\omega)d\bm\omega=1.
    \end{align}
    Equivalently, for smooth solutions, \(\partial_{\bm n}\rho_t+\beta\rho_t\partial_{\bm n}U_t=0\) on \(\partial\Omega\). Here \(V(\bm\omega)=-\mathbb{E}[yy'\phi_{\bm\omega}(\bm{x})\phi_{\bm\omega}(\bm{x}')]\) is the external potential induced by the kernel-target alignment loss.
\end{theorem}

The boundary condition in \eqref{eq:fokker-planck-boundary} is the multidimensional analogue of the Robin condition: it says that the probability flux through the boundary is zero. In the zero-temperature limit $(\beta\to\infty)$, the evolution reduces to the reflected deterministic continuity equation
\begin{align}
    \frac{\partial \rho_{t}(\bm{\omega})}{\partial t} + \nabla_{\bm{\omega}} \cdot \left( \rho_{t}( \bm{\omega}) \bm{v}_{t}(\bm{\omega}) \right) = 0,
    \qquad
    (\rho_t\bm v_t)\cdot\bm n=0\quad\text{on }\partial\Omega,
\end{align}
where $\bm{v}_{t}(\bm{\omega}) = -\nabla_{\bm{\omega}} U_t(\bm\omega)$. This PDE formalism captures the macroscopic evolution of particle densities driven by repulsion, external alignment forces, reflection, and thermal fluctuations.

At equilibrium, for any fixed $\beta_{N} > 0$ and $N>0$, the particle configuration is distributed according to the \textit{canonical Gibbs measure}:

\begin{align}
    \label{Eq:Gibbs_Measure}
     \mathbb{P}_{N,\beta_{N}}(\mathrm{d}\bm{\Omega}_{N}) = \dfrac{1}{Z_{N,\beta_{N}}} \exp\left(-\beta_{N} \mathcal{H}_{N}(\bm{\Omega}_{N})\right) \bm{1}_{\Omega^{N}}(\bm{\Omega}_{N}) \, \mathrm{d}\bm{\Omega}_{N},
\end{align}
where $\bm{1}_{\Omega^{N}}(\bm{\Omega}_{N})$ is the indicator function of $\Omega^{N}$, $\mathrm{d}\bm{\Omega}_{N}$ is the Lebesgue measure on $(\mathbb{R}^{d})^{N}$, and $Z_{N,\beta}$ is the \textit{partition function},
\begin{align}
    Z_{N,\beta_{N}} \defeq \int_{\Omega^{N}} \exp\left(-\beta_{N} \mathcal{H}_{N}(\bm{\Omega}_{N})\right) \, \mathrm{d}\bm{\Omega}_{N},
\end{align}
which ensures that the Gibbs measure is properly normalized, i.e., $\int_{\Omega^{N}} \mathbb{P}_{N,\beta}(\bm{\Omega}_{N}) \, \mathrm{d}\bm{\Omega}_{N} = 1$. Moreover, 
the \textit{free energy} is defined as 
\begin{align}
    F_{s}(\beta_{N},N,\Omega) \defeq -\dfrac{1}{\beta_{N}}\log Z_{N,\beta_{N}}.
\end{align}
As $\beta_{N} \to +\infty$ with $N\rightarrow \infty$, the Gibbs measure increasingly concentrates around the minimizer (ground state) of the Hamiltonian $\mathcal{H}_{N}(\bm{\Omega}_{N})$, which corresponds to the solution of optimization problem in Eq. \eqref{Eq: Empirical Kernel-Target Align}.

\subsection{Thermodynamic limit of the equilibrated state in the short-range case, \( s > d \)}

Since the kernel function is approximated by the empirical measure $\mu_N$ through the Monte Carlo approximation in Eq.~\eqref{Eq:Monte_Carlo}, it is essential to analyze the asymptotic behavior of this empirical measure in order to characterize the limiting properties of the kernel approximation itself. In particular, the statistical fluctuations and concentration properties of $\mu_N$ directly determine the accuracy and stability of the resulting kernel-based quantities.

In the setting of interest in Algorithm \ref{alg:langevin_kernel}, the empirical measure $\mu_N$  arises from a system of interacting \textit{charged} particles evolving under Langevin dynamics. Once the dynamics have reached equilibrium (i.e., after a sufficiently large number $m$ of iterations), the distribution of the particle system converges to the canonical Gibbs measure given in Eq.~\eqref{Eq:Gibbs_Measure}. This measure describes the statistical equilibrium of the system, incorporating both the deterministic interaction potential and the stochastic perturbations induced by thermal noise. From the perspective of statistical mechanics, such an equilibrium corresponds to a thermodynamically stable macroscopic state, in which relevant observables become stationary in distribution.

The principal objective of this section is to investigate the asymptotic behavior of the random empirical measure $\mu_{N}$ in this equilibrium regime, particularly as the number of particles $N$ tends to infinity. This regime, known as the \emph{thermodynamic limit} ($N \to \infty$), is of fundamental importance in both statistical mechanics and probability theory, as it establishes the connection between microscopic particle interactions and macroscopic statistical laws.

Our main result is the derivation of a \emph{large deviation principle} (LDP) for $\mu_{N}$, which characterizes the exponential decay of probabilities of rare deviations from the typical equilibrium distribution. The LDP is governed by a \emph{rate function}, which assigns to each admissible probability measure a nonnegative ``cost'' quantifying the likelihood of its occurrence in the large-$N$ limit. This framework provides a precise quantitative description of the concentration of $\mu_{N}$ around its equilibrium value, as well as the nature of its fluctuations.

For a comprehensive background on the theory of large deviations, we refer the reader to \cite{DZ10} and \cite{Var16}. For completeness, we recall the formal definition below, which introduces the notion of an LDP and the associated rate function.

\begin{definition}[Large Deviation Principle (LDP)]
    \label{Definition: Large Deviation Principle}
    Let \((\nu_N)_{N \geq 1}\) be a sequence of probability measures on a Polish space \(\Omega\) equipped with the Borel $\sigma$-algebra $\mathcal{B}(\Omega)$. We say that \((\nu_N)\) satisfies a \textit{Large Deviation Principle} (LDP) at speed \(r_N\) with rate function \(\mathcal{I} : \Omega \to \mathbb{R}_{+}\) if, for every Borel set \(B \subset \mathcal{B}(\Omega)\),
    \[
    -\inf_{x \in \mathring{B}} \mathcal{I}(x) \leq \liminf_{N \to \infty} \frac{1}{r_N} \log \nu_N(B) \leq \limsup_{N \to \infty} \frac{1}{r_N} \log \nu_N(B) \leq -\inf_{x \in \overline{B}} \mathcal{I}(x),
    \]
    where \(\mathring{B}\) and \(\overline{B}\) denote the interior and closure of \(B\), respectively. The functional \(\mathcal{I}\) is called a \textit{good rate function} if it is lower semi-continuous and has compact sub-level sets. 
\end{definition}

Equipped with Definition \ref{Definition: Large Deviation Principle}, we are ready to state the following theorem:

\begin{theorem}[Large deviations for empirical measures]
    \label{Theorem: LDP for Empirical Measures}
    Let $\Omega \subset \mathbb{R}^d$ be a bounded open set. For each $N\ge 1$, let
    $\{\bm{\omega}_1,\dots,\bm{\omega}_N\}\subset \Omega$ be a random configuration with law
    $\mathbb{P}_{N,\beta_N}$ as in \eqref{Eq:Gibbs_Measure}, and define the empirical measure
    \[
    \mu_N \;=\; \frac1N \sum_{k=1}^N \delta_{\bm{\omega}_k}\;\in\;\mathcal{P}(\Omega).
    \]
    Then $\{\mu_N\}_{N\ge 1}$ satisfies a large deviation principle on $\mathcal{P}(\Omega)$
    endowed with the weak topology, with speed $r_N$ and good rate function given as follows.
    
    Define the energy functional
    \[
    \mathcal{E}_s(\mu)
    \defeq
    \int_{\Omega} \phi_{\bm{\omega}}(\bm{x}_i)\,\phi_{\bm{\omega}}(\bm{x}_j)\,\mu(\mathrm{d}\bm{\omega})
    \;+\;
    \frac{\lambda}{2}\int_{\Omega}\int_{\Omega}
    g_s(\bm{\omega}-\bm{\omega}')\,\mu(\mathrm{d}\bm{\omega})\,\mu(\mathrm{d}\bm{\omega}'),
    \]
    and the (relative) entropy with respect to Lebesgue measure $\ell$ on $\Omega$,
    \[
    \mathrm{Ent}(\mu\mid \ell)
    \defeq
    \begin{cases}
        \displaystyle
        \int_{\Omega} \log\!\Big(\frac{\mathrm{d}\mu}{\mathrm{d}\ell}(\bm{\omega})\Big)\,\mu(\mathrm{d}\bm{\omega})
        =
        \int_{\Omega} \frac{\mathrm{d}\mu}{\mathrm{d}\ell}(\bm{\omega})
        \log\!\Big(\frac{\mathrm{d}\mu}{\mathrm{d}\ell}(\bm{\omega})\Big)\,\ell(\mathrm{d}\bm{\omega}),
        & \mu\ll \ell,\\[2.0ex]
        +\infty, & \text{otherwise.}
    \end{cases}
    \]
    
    \begin{itemize}
        \item If $\beta_N/N \to 1$ as $N\to\infty$, then the LDP holds with speed $r_N=N$ and rate function
        \[
        \mathcal{J}_s(\mu)
        =
        \Big(\mathcal{E}_s(\mu)+\mathrm{Ent}(\mu\mid \ell)\Big)
        -\inf_{\nu\in\mathcal{P}(\Omega)}\Big(\mathcal{E}_s(\nu)+\mathrm{Ent}(\nu\mid \ell)\Big).
        \]
        
        \item If $\beta_N/N \to \infty$ as $N\to\infty$, then the LDP holds with speed $r_N=\beta_N$ and rate function
        \[
        \mathcal{J}_s(\mu)
        =
        \mathcal{E}_s(\mu)
        -\inf_{\nu\in\mathcal{P}(\Omega)}\mathcal{E}_s(\nu).
        \]
    \end{itemize}
\end{theorem}

The proof is provided in Appendix \ref{subsection: LDP}. Theorem~\ref{Theorem: LDP for Empirical Measures} implies that for any Borel set
$A\subset \mathcal{P}(\Omega)$,
\begin{align*}
    -\inf_{\mu\in A^\circ} \mathcal{J}_s(\mu)
    \;\le\;
    &\liminf_{N\to\infty}\frac{1}{r_N}\log \mathbb{P}_{N,\beta_N}(\mu_N\in A)
     \\
    &\;\le\; \limsup_{N\to\infty}\frac{1}{r_N}\log \mathbb{P}_{N,\beta_N}(\mu_N\in A)
    \;\le\;
    -\inf_{\mu\in \overline{A}} \mathcal{J}_s(\mu),
\end{align*}
where $A^\circ$ and $\overline{A}$ denote the interior and closure of $A$, respectively. The theorem states that, under the Gibbs law, the empirical measure $\mu_N$ concentrates (in the weak topology) around the minimizers of the relevant variational functional, and the probability of observing a macroscopic deviation decays exponentially fast at speed $r_N$. In the regime $\beta_N \sim N$, the rate function contains both the interaction energy $\mathcal{E}_s(\mu)$ and the entropy term $\mathrm{Ent}(\mu\mid \ell)$, capturing the competition between energetic preference for structured configurations and entropic preference for spreading mass. In the low--temperature regime $\beta_N/N \to \infty$, the entropy contribution becomes negligible, so the large deviations are governed purely by $\mathcal{E}_s$ and $\mu_N$ concentrates on energy minimizers, with fluctuations suppressed on the faster exponential scale set by $\beta_N$.


\begin{figure}[t]
	\centering
	\safeincludegraphics[width=0.7\linewidth]{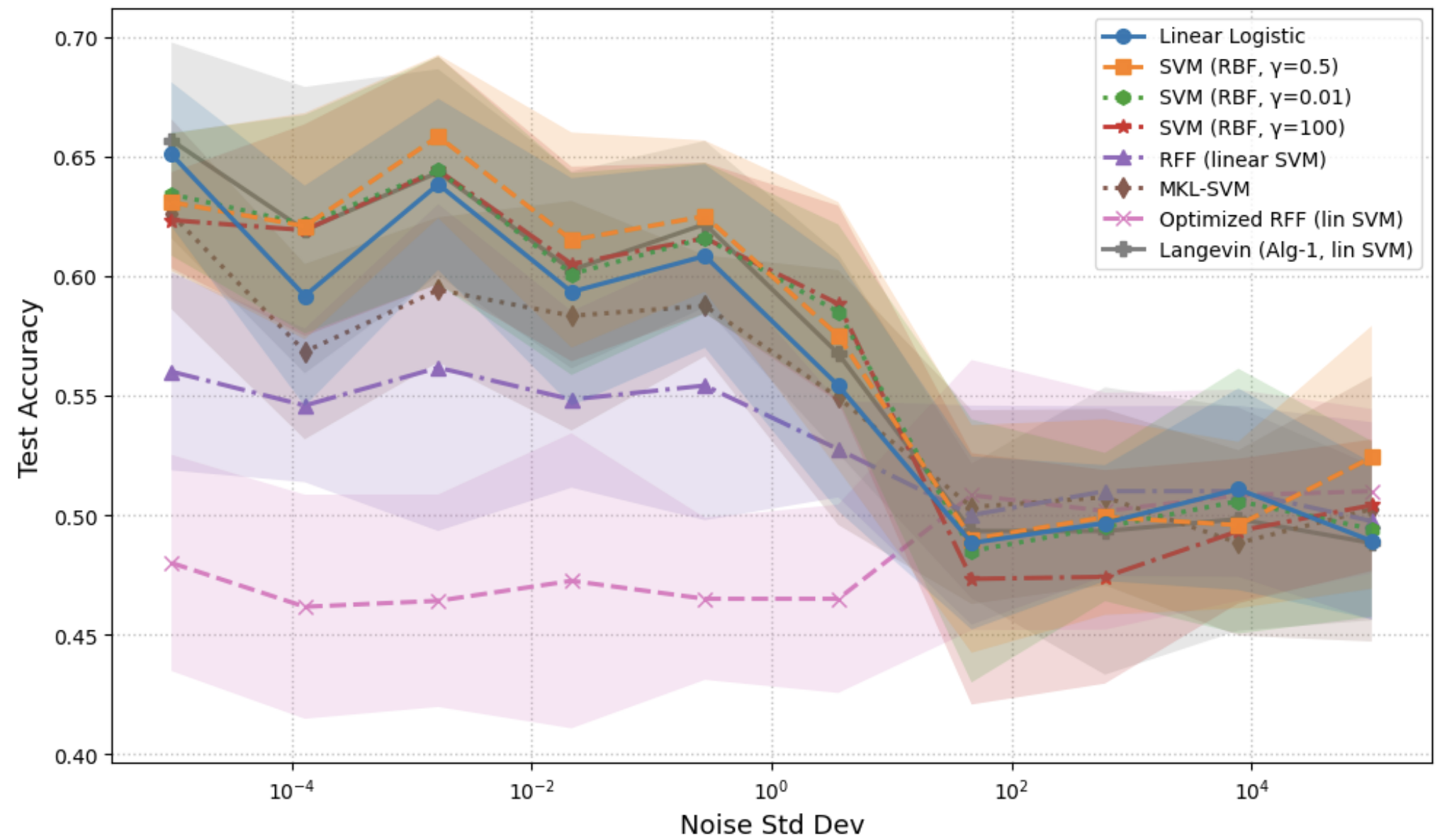}\hfill
	\safeincludegraphics[width=0.7\linewidth]{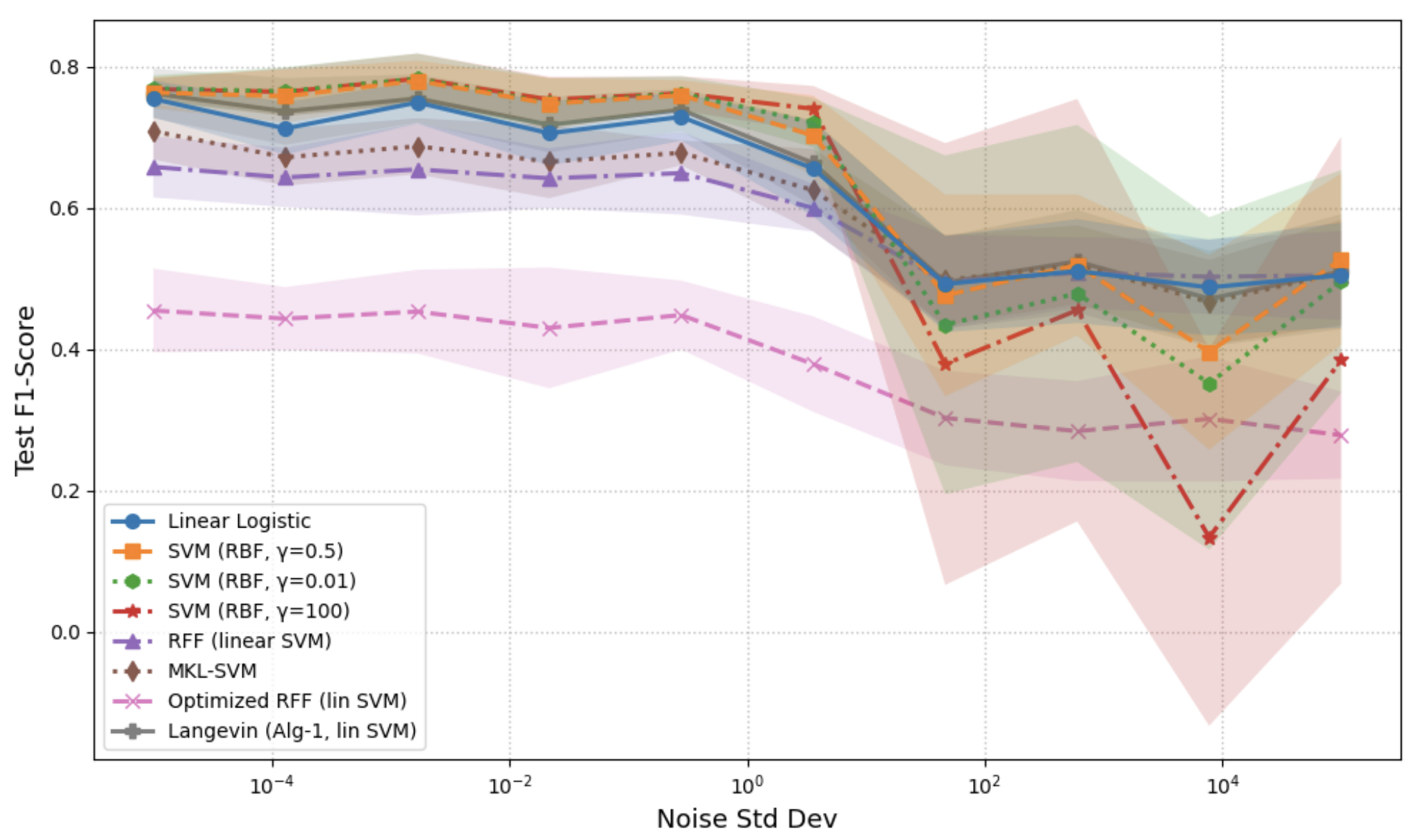}
	\caption{Accuracy (left) and F1 score (right) versus noise standard deviation $\sigma^2$ for kernel-learning approaches. Confidence intervals summarize variability across 10 independent trials.}
	\label{fig:supp-noise-sweep}
\end{figure}

\begin{table}[t]
	\centering
	\small
	\caption{Observed wall-clock runtime per job in \textit{seconds}. Values correspond only to measured sample sizes.}
	\label{tab:runtime-observed}
	\resizebox{\textwidth}{!}{%
		\begin{tabular}{@{} l r r r r r r r r @{}}
			\toprule
			\textbf{Model} 
			& \multicolumn{8}{c}{\textbf{Number of samples } $n$} \\
			\cmidrule(l){2-9}
			& $100$ & $500$ & $1{,}000$ & $2{,}000$ & $3{,}000$ & $4{,}000$ & $10{,}000$ & $20{,}000$ \\
			\midrule
			Linear Logistic 
			& 0.142 & 0.100 & 0.0882 & 0.184 & 0.173 & 0.204 & 0.117 & 0.121 \\
			SVM (RBF kernel) 
			& 0.00231 & 0.0425 & 0.140 & 0.0910 & 0.317 & 0.468 & 1.98 & 9.72 \\
			RFF (linear SVM) 
			& 0.00332 & 0.0250 & 0.0270 & 0.0326 & 0.0416 & 0.0571 & 0.160 & 0.335 \\
			MKL-SVM 
			& 2.06 & 9.14 & 29.6 & 131 & 482 & 874 & 7.21e3 & 4.49e4 \\
			Optimized RFF (lin SVM) 
			& 0.0346 & 0.127 & 0.299 & 0.533 & 0.855 & 1.04 & 2.94 & 4.87 \\
			Langevin (Alg-1, lin SVM) 
			& 21.6 & 26.1 & 45.8 & 80.5 & 95.9 & 124 & 417 & 922 \\
			\bottomrule
		\end{tabular}%
	}
\end{table}

\begin{table}[t]
    \centering
    \small
    \caption{Power-law fits $t(n)=c\,n^{\gamma}$ summarizing runtime scaling behavior. Exponent $\gamma$ measures the growth rate, prefactor $c$ sets the baseline runtime, and $\sigma_{\log}$ indicates variability in log-space residuals.}
    \label{tab:powerlaw-summary}
    \begin{tabular}{@{} l r r r c @{}}
        \toprule
        \textbf{Model} & \multicolumn{1}{c}{$\gamma$} & \multicolumn{1}{c}{$c$} & \multicolumn{1}{c}{$\sigma_{\log}$} & $n_{\text{obs}}$ \\
        \midrule
        MKL-SVM & 1.961 & 8.00e-05 & 0.671 & 8 \\
        SVM (RBF kernel) & 1.477 & 3.00e-06 & 0.503 & 8 \\
        Optimized RFF (lin SVM) & 0.967 & 3.55e-04 & 0.085 & 8 \\
        RFF (linear SVM) & 0.799 & 9.50e-05 & 0.319 & 8 \\
        Langevin (Alg-1, lin SVM) & 0.730 & 3.82e-01 & 0.430 & 8 \\
        Linear Logistic & 0.021 & 1.16e-01 & 0.290 & 8 \\
        \bottomrule
    \end{tabular}

\end{table}

\begin{figure}[t]
	\centering
	\safeincludegraphics[width=0.5\linewidth]{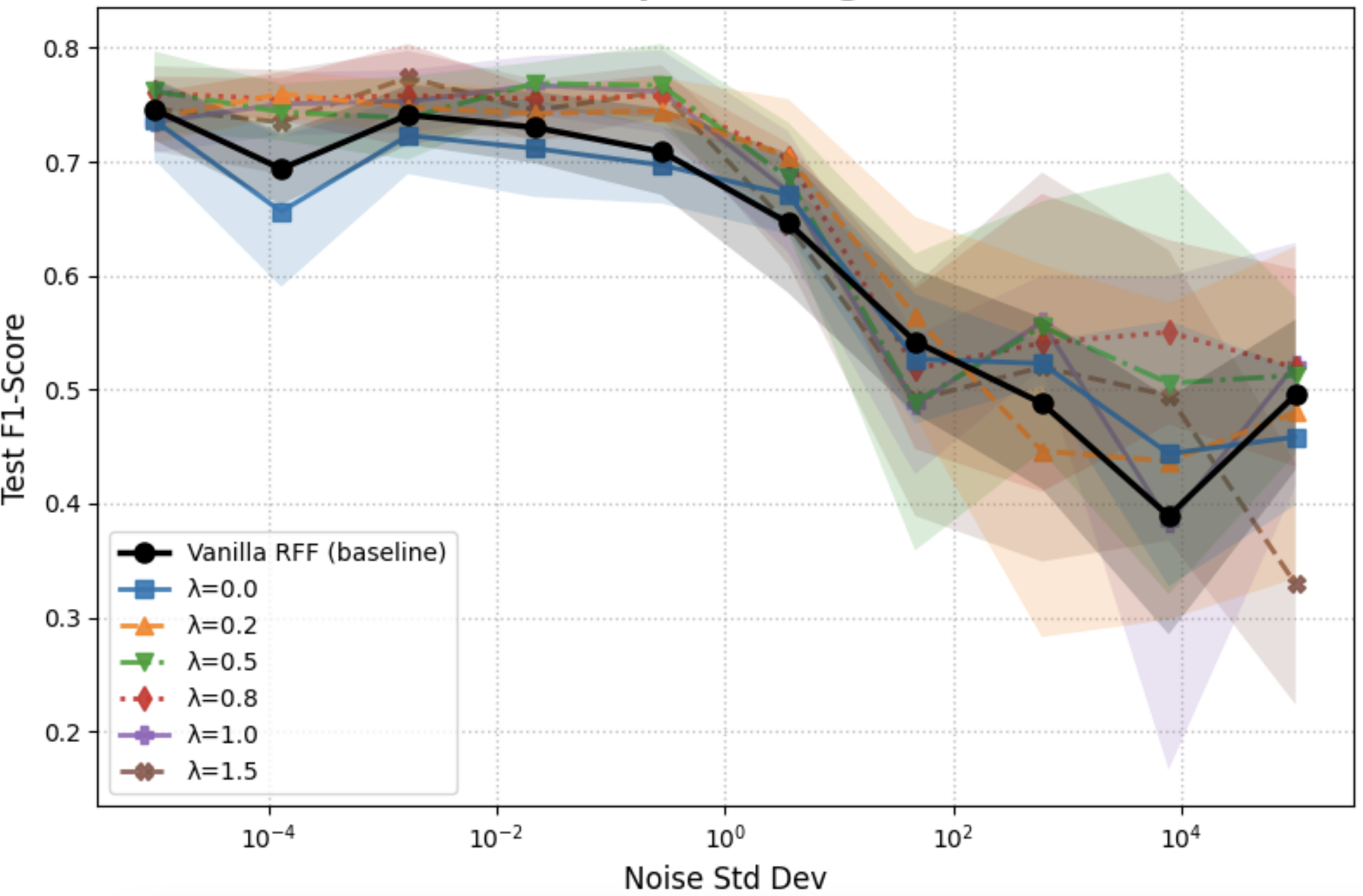}\hfill
	\safeincludegraphics[width=0.5\linewidth]{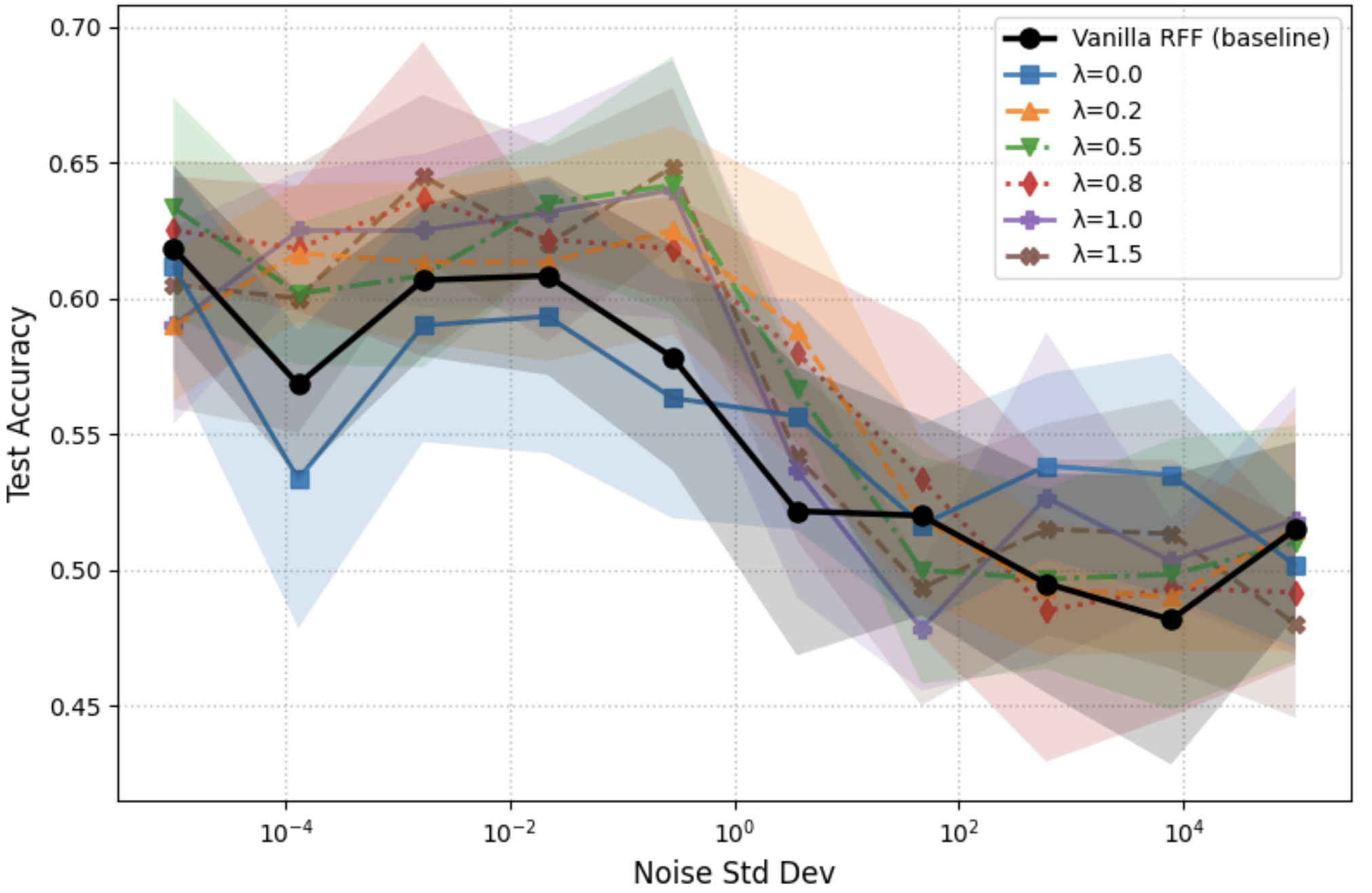}\\
	\safeincludegraphics[width=0.5\linewidth]{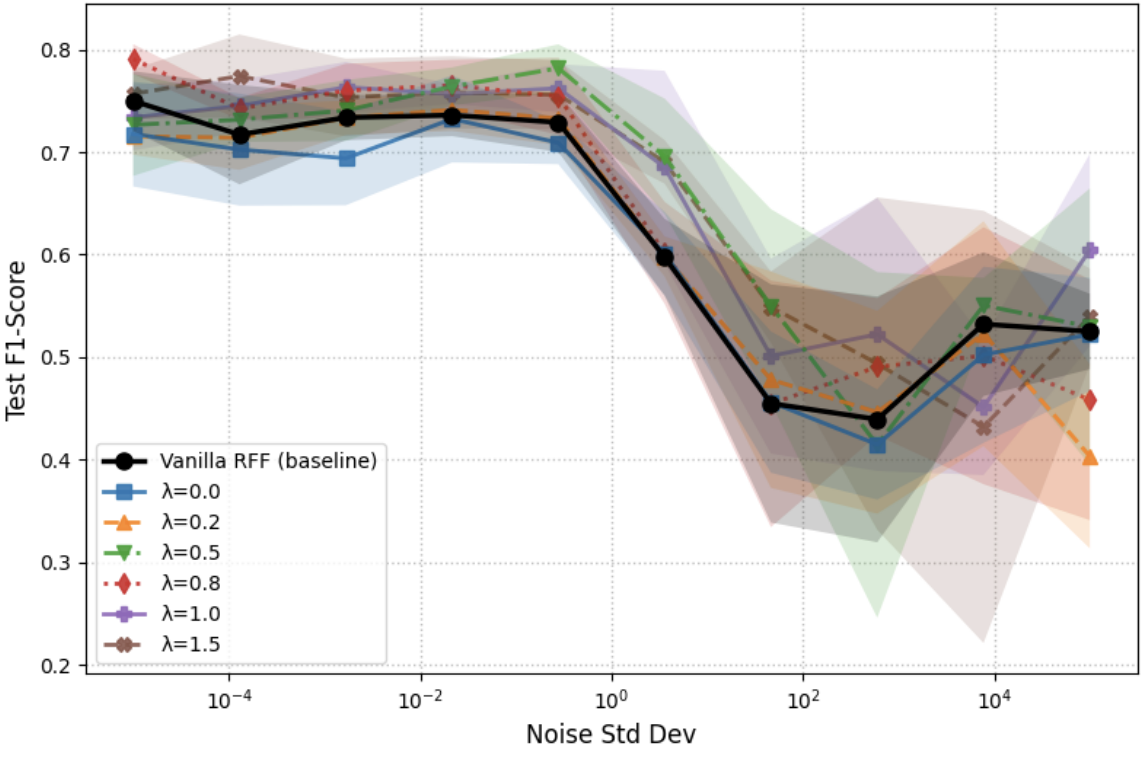}\hfill
	\safeincludegraphics[width=0.5\linewidth]{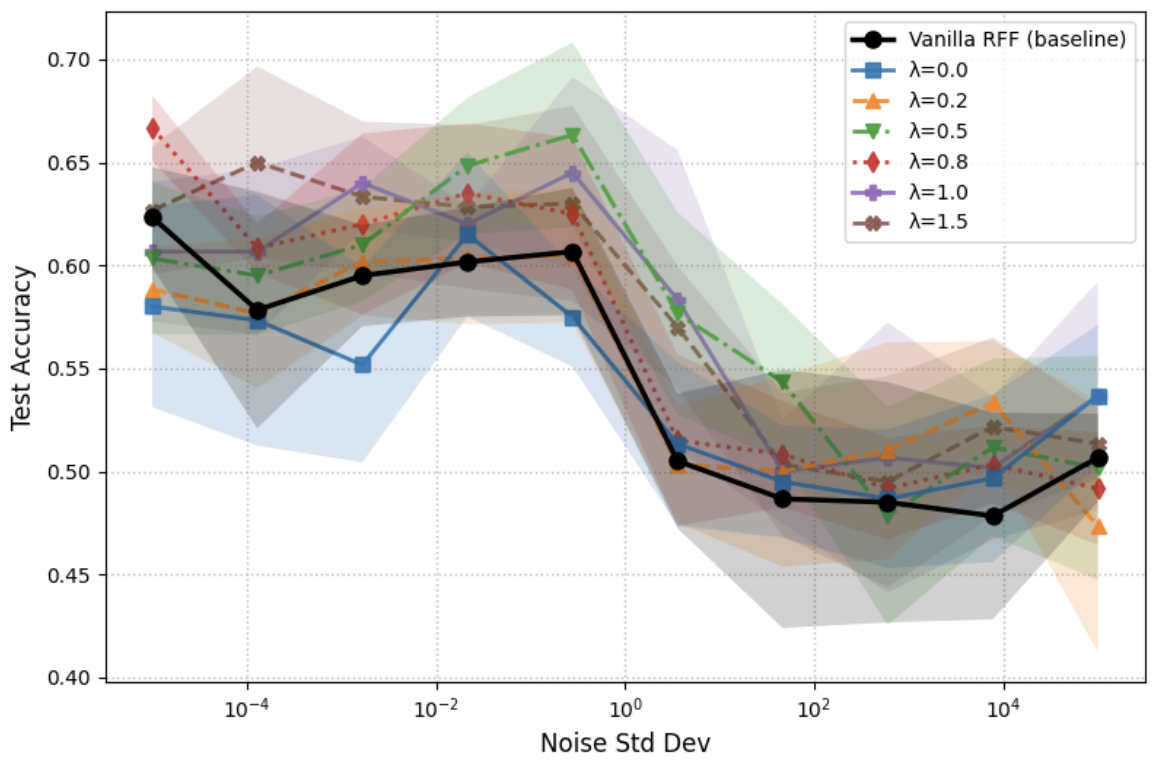}\\
	
	\caption{Ablation study of Coulomb (top) and Riesz (bottom) potentials for charged particles. Increasing $\lambda$ strengthens the repulsive interaction between particles. The baseline uses vanilla RFF with no sampling optimizations. Confidence intervals summarize variability across 10 independent trials. }
	\label{fig:supp-lambda-sweep}
\end{figure}

\section{Numerical Experiments}
We present simulations on synthetic datasets as well as experiments on NLP datasets. We plan to release the code for all experiments publicly upon publication.

\subsection{Nonlinear synthetic classification: data, estimators, and evaluation protocol}
\label{subsec:nonlinear-synth-cls}

\subsubsection{Problem setting.}
We consider binary classification with a fixed design matrix \( \bm{X}\in\mathbb{R}^{n\times p} \) drawn once as \(X_{ij}\overset{\mathrm{i.i.d.}}{\sim}\mathsf{N}(0,1)\) under a global seed. Unless stated otherwise we use \(n=400\) and a configurable feature dimension \(p\ge 5\). Writing \(\bm{x}_i\) for the \(i\)-th row of \(\bm{X}\), the  log-odds is
\[
\ell(\bm{x}_i)
= 1.5\,\sin(\pi\,x_{i,1})
\;+\;0.8\,x_{i,2}^{\,2}
\;-\;1.0\,x_{i,3}\,x_{i,4}
\;+\;0.5\,\sin(3\,x_{i,5})
\;+\;\mathbf{1}_{\{p>5\}}\;\bm{x}_{i,6:p}^{\!\top}\bm{w},
\]
where the ``extra--dimensions'' weight vector \(\bm{w}\in\real^{p-5}\) is drawn once from \(\mathsf{N}(\bm{0},0.3^2\bm{I})\) (same \(\bm{w}\) for all trials), and $\bm{1}_{\{\cdot\}}$ is the indicator function. Labels are generated by corrupting the logit with additive Gaussian noise and passing through a logistic link:
\[
\xi_i \sim \mathsf{N}(0,\sigma^2),\qquad
y_i \,\big|\, \bm{x}_i \sim \mathrm{Bernoulli}\!\left( \varsigma\!\big(\ell(\bm{x}_i)+\xi_i\big) \right),
\quad \varsigma(t)=\tfrac{1}{1+e^{-t}}.
\]
We sweep the \emph{logit noise standard deviation} \(\sigma\) over logarithmically spaced values in \([10^{-5},10^{5}]\). A single \(70\%\)/\(30\%\) train/test split (random state \(0\)) of \(\bm{X}\) is reused across all trials and noise levels; for each \(\sigma\) and trial, only the noise \((\xi_i)\) (and any method-specific randomness) is resampled.

\subsubsection{Baselines and learned feature maps.}
All methods use the same train/test partitions of \(\bm{X}\). Random-feature methods are aligned to a common feature budget \(D=200\).
\begin{enumerate}[label=(\arabic*), leftmargin=*, itemsep=2pt, topsep=2pt]
    \item \textit{Linear logistic regression} on the raw inputs \(\bm{X}\) (baseline).
    \item \textit{SVM with Gaussian (RBF) kernel} \(K(\bm{x},\bm{x}')=\exp\!\big(-\gamma\lVert\bm{x}-\bm{x}'\rVert_2^2\big)\) with fixed \(\gamma=0.5\) and default hinge loss, \(C=1\).
    \item \textit{Random Fourier features (RFF) + linear SVM.}
    Draw \(W_{:,j}\!\sim\!\mathsf{N}(0,\,2\gamma\,\bm{I}_p)\), \(b_j\!\sim\!\mathrm{Unif}[0,2\pi]\),
    form \( z(\bm{x})=\sqrt{2/D}\,\cos(W^\top\bm{x}+b) \) with optimized \(\gamma=0.5\) and \(D=200\),
    then train a linear SVM (hinge loss, \(C=1\)).
    \item \textit{Multiple-kernel learning (MKL) + SVM (precomputed).}
    Build Gaussian Gram matrices with \(\gamma\in\{0.1,0.3,0.7,1.0,1.3,1.6\}\), apply the
    \textsc{MKLpy} normalization, learn nonnegative mixture weights via \textsc{MEMO},
    and feed the weighted train/test kernels to an SVM with a precomputed kernel.
    \item \textit{Optimized random features (importance sampling) + linear SVM.}
    Following \cite{sinha2016learning}, sample \(N_w=10^4\) candidate features, use
    divergence threshold \(\rho=N_w\cdot 0.005\) and tolerance \(10^{-10}\) to obtain a
    reweighted set, then align to \(D=200\) via stratified subsampling (padding if needed).
    Train a linear SVM on the resulting features (\(C=1\)).
    \item \textit{Langevin spectral estimators (two variants) + linear SVM.}
    Evaluate (i) an ``Alg-1'' Coulomb-gas variant with \(\lambda_{\mathrm{reg}}=0.5\) and
    (ii) a robust variant with \(\lambda_{\mathrm{reg}}=0\). Both use \(N=300\) particles,
    feature budget \(D=200\), step size \(\eta=30\), inverse temperature \(\beta=10^2\),
    horizon \(T_{\max}=2000\), logarithmic repulsion, scaling by \(N\), maximum frequency
    norm \(5.0\), and gradient clipping \(1.0\). After fitting, we form Fourier features
    from the learned frequencies/phases and train a linear SVM (\(C=1\)).
\end{enumerate}

\subsubsection{Implementation.}
We rely on \textsc{MKLpy} \cite{lauriola2020mklpy} for MKL (mixture-of-RBFs) and on scikit-learn for linear/logistic models and SVMs. The importance-sampling optimized RFF pipeline of \cite{sinha2016learning} is reimplemented in Python. The Langevin estimators are also implemented in Python. We run the simulations on CPU.

\subsubsection{Training protocol, metrics, and plots.}
For each noise level \(\sigma\) and each of three independent trials, we fit every available method on the training split and evaluate on the test split. We report, for each method and \(\sigma\), the mean and standard deviation (over trials) of:
\begin{itemize}[leftmargin=*, itemsep=1pt]
    \item \textit{Test accuracy}: \( \mathrm{Acc} = \tfrac{1}{n_{\mathrm{test}}}\sum_{i\in\mathcal{D}_{\mathrm{test}}}\mathbf{1}\{\hat{y}_i = y_i\}\).
    \item \textit{Test F1-score}: (binary, with \( \mathrm{zero\_division}=0\) in implementation).
    \item \textit{Wall-clock time per job}: (seconds), measured with a timer from the start of a method's pipeline (feature construction included for RFF variants) to prediction on the test set.
\end{itemize}
We produce three summary plots versus \(\sigma\) (logarithmic x-axis): (A) accuracy, (B) F1-score, and (C) runtime; shaded bands depict \(\pm1\) standard deviation. 

\subsubsection{Reproducibility and computation.}
A global seed (\(42\)) fixes \(\bm{X}\), the train/test split, and (when \(p>5\)) the extra-dimension weights \(\bm{w}\); per-job seeds are derived deterministically from the noise and trial indices. We cap BLAS threads to one in parent and worker contexts to avoid oversubscription. Jobs indexed by \((\sigma,\mathrm{trial})\) are executed in parallel using a process pool with up to eight workers; progress is tracked asynchronously. Across methods using random features, the feature budget is held fixed at \(D=200\) to ensure comparable capacity.

\subsubsection{Results.} In Figure~\ref{fig:supp-noise-sweep}, we report accuracy and F1 score as functions of the additive-noise standard deviation $\sigma$. Across noise levels, the SVM with a Gaussian (RBF) kernel---when its bandwidth $\gamma$ is properly tuned---achieves the best performance. This is expected because the random-feature model only approximates the kernel; with a small number of features ($D=200$), approximation error degrades classification performance. The same figure also shows that the RBF SVM can underperform substantially when the bandwidth $\gamma$ is mis-specified, e.g., $\gamma=0.01$ and $\gamma=100$ corresponding to green and red dashed lines, respectively.

Moreover, relative to a linear SVM trained on fixed random Fourier features (RFF), jointly optimizing the random features yields clear gains in both accuracy and F1 (purple dashed vs.\ grey solid curves), confirming the benefit of random-feature optimization. Moreover, compared to alternative importance sampling of random features, we obtain a clear gain by optimization of random feature samples directly (brown line vs purple line).

\subsubsection{Ablation: Coulomb/Riesz interaction strength.} We vary the potential strength \(\lambda\in\{0.0,0.2,0.5,0.8,1.0,1.5\}\) and evaluate robustness as input noise increases (x--axis: noise standard deviation on a log scale). Across both metrics (Accuracy and F1), small noise regimes (\(\sigma\!\le\!10^{-3}\)) show minimal separation between curves, indicating that the potential has little effect when the signal is clean. As noise grows to the moderate range (\(10^{0}\!\!-\!10^{2}\)), nonzero potentials consistently outperform the vanilla RFF baseline and \(\lambda=0\), with \(\lambda\approx 0.5\!-\!1.0\) yielding the most reliable gains (typically a few points) and smaller variance, suggesting improved stability. In the extreme-noise regime (\(\sigma\!\ge\!10^{3}\)), performance degrades for all settings and the gaps narrow; very large strengths (\(\lambda\!=\!1.5\)) can oversmooth and underperform, while small-to-moderate \(\lambda\) remains competitive but offers diminishing returns. Overall, a \emph{moderate} Coulomb/Riesz potential (\(\lambda\approx 0.5\!-\!1.0\)) provides the best robustness--accuracy trade-off, whereas too weak or too strong potentials are less effective.

\subsection{Sentence-Level Classification Benchmarks}
\label{sec:sentcls}

\begin{table}[t]
    \centering
    \caption{Statistics of sentence-level and sentence-pair classification datasets used in our experiments. Counts denote the number of examples per split. For sentence-pair tasks, each example consists of a pair of sentences.}
    \label{tab:sentcls_datasets}
    \begin{tabular}{lccccc}
        \toprule
        \textbf{Dataset} & \textbf{Train} & \textbf{Dev} & \textbf{Test} & \textbf{Classes} & \textbf{Task} \\
        \midrule
        SST-2 \cite{socher2013recursive}          & 67{,}349  & 872    & 1{,}821   & 2 & Sentiment classification \\
        QQP \cite{wang2018glue}                   & 363{,}846 & 40{,}430 & 390{,}965 & 2 & Duplicate question detection \\
        Rotten Tomatoes \cite{pang2005seeing}    & 8{,}530   & 1{,}066 & 1{,}066   & 2 & Sentiment classification \\
        \bottomrule
    \end{tabular}
\end{table}

\paragraph{Dataset selection.}
We evaluate sentence-level and sentence-pair classification performance on a set of widely used English benchmarks that are standard in prior work on attention mechanisms.

The selected datasets span multiple semantic phenomena and dataset scales, including binary sentiment analysis on \textit{SST-2} and \textit{Rotten Tomatoes}, as well as semantic equivalence detection on the \textit{QQP} benchmark. Together, these tasks cover both single-sentence and sentence-pair settings, enabling evaluation across varying supervision regimes and levels of linguistic complexity.

\paragraph{Datasets and splits.}
\textit{SST-2}~\cite{socher2013recursive} is a binary sentiment classification task derived from the Stanford Sentiment Treebank. We use the GLUE version of the dataset, which removes neutral examples and provides sentence-level annotations, and follow the standard GLUE train/dev/test splits.

\textit{QQP}~\cite{wang2018glue} (Quora Question Pairs) is a sentence-pair classification task in which each example consists of two questions from Quora annotated for semantic equivalence. We adopt the official GLUE splits and process each question pair by concatenation with a special separator token.

\textit{Rotten Tomatoes}~\cite{pang2005seeing} is a sentence-level binary sentiment classification dataset constructed from movie reviews. We use the standard polarity version of the dataset and follow the commonly used train/dev/test splits.

\paragraph{Preprocessing.}
We apply minimal preprocessing to preserve the original linguistic structure of each dataset. All text is tokenized using a fixed WordPiece vocabulary, with original casing and punctuation retained. Inputs are truncated or padded to a maximum length of 128 tokens. For sentence-pair inputs in \textit{QQP}, the two questions are concatenated using a special separator token. No segment embeddings or task-specific features are used unless otherwise stated.

\paragraph{Training protocol.}
For each dataset, we train models on the official training split and use the validation split exclusively for model selection and early stopping. Hyperparameters are fixed across all experiments and are not tuned on validation data. For each method, we select the checkpoint that achieves the highest validation accuracy and report performance on the held-out test split using this checkpoint. Test data are never used for hyperparameter selection or model selection.

\paragraph{Evaluation and metrics.}
Across all datasets, we report classification accuracy as the primary performance metric. For binary classification tasks (SST-2, QQP, and Rotten Tomatoes), we additionally report micro-, macro-, and weighted $F_1$ scores to account for potential class imbalance. To evaluate ranking quality and probabilistic reliability, we also report ROC-AUC, precision--recall AUC, Matthews correlation coefficient (MCC), balanced accuracy, LogLoss, Brier score, and expected calibration error (ECE). All metrics are computed on the test split using the checkpoint selected based on validation performance.

\paragraph{Experimental setup.}
Inputs are tokenized using the \texttt{bert-base-uncased} WordPiece tokenizer and truncated or padded to a maximum sequence length of 128 tokens. Models use a two-layer Transformer encoder with hidden size 128, two attention heads, feedforward dimension 256, and $m=256$ random features per head. In our particle-based view, each column of $\bm{\Omega}_{N}\in\mathbb{R}^{H\times d_k \times m}$ is a particle, so this corresponds to 256 particles per head (512 per layer; 1024 particles total across the two-layer encoder). Dropout is set to 0.1, and mean pooling is applied over token representations. We use a $\log$ repulsion force between particles.

Training uses a batch size of 64. When kernel optimization is enabled, training proceeds in two phases. In Phase~A, the attention feature map parameters $\bm{\Omega}_{N}$ (and optionally layer normalization and value projection parameters) are optimized using an alignment loss objective for up to 10 epochs via particle optimization (SGLD/Langevin dynamics with step size $\eta=2\times10^{-3}$ and inverse temperature $\beta=50$, i.e., Gaussian noise scale $\sqrt{2\eta/\beta}$, plus a repulsion term with $\lambda=10^{-3}$ and gradient clipping at norm 10), while the remaining model parameters are frozen. Column-wise $L_2$ norms of $\bm{\Omega}_{N}$ (i.e., per-particle norms) are constrained to 1.5. Training in Phase~A stops early if the maximum change in $\bm{\Omega}_{N}$ falls below $10^{-6}$. In Phase~B, the kernel parameters are frozen and the model is trained end-to-end with cross-entropy loss for 10 epochs. For vanilla baselines, only Phase~B is used. Optimization in Phase~B is performed with Adam using a learning rate of $2\times10^{-4}$.






\paragraph{Effect of Kernel Learning.}
Tables~\ref{tab:sst2_detailed_metrics}, \ref{tab:qqp_detailed_metrics}, and \ref{tab:rotten_tomatoes_detailed_metrics} quantify the effect of target-alignment kernel learning across SST-2, QQP, and Rotten Tomatoes. On SST-2 (Table~\ref{tab:sst2_detailed_metrics}), where overall performance is near saturation, kernelization yields feature-map-dependent gains. Kernel-porf-softplus attains the best accuracy and MCC (0.8188/0.6376). Several kernel variants improve both discriminative and probabilistic metrics: for example, Kernel-elu increases accuracy from 0.8028 to 0.8108 while simultaneously reducing LogLoss (0.4938$\rightarrow$0.4502) and Brier score (0.1505$\rightarrow$0.1392), and Kernel-softmaxfeat improves accuracy (0.8005$\rightarrow$0.8085) while achieving near-optimal LogLoss (0.4381). Calibration-specific improvements are also evident: Kernel-cos2 yields the lowest Brier score (0.1366), while Kernel-favor achieves the lowest ECE (0.3086). However, not all kernels are uniformly beneficial---Kernel-softplus degrades both Brier and ECE---highlighting sensitivity to the choice of feature map.

On the more challenging QQP task (Table~\ref{tab:qqp_detailed_metrics}), kernel learning provides more consistent benefits across feature maps. Kernel-softmaxfeat is the strongest $Q{=}K$ model (Acc 0.7909, MCC 0.5506), with substantial reductions in LogLoss (0.4625$\rightarrow$0.4354) and Brier score (0.1512$\rightarrow$0.1418), while Kernel-favor achieves the best calibration as measured by ECE (0.4106). Although kernelization roughly doubles training time, it does not affect inference-time complexity, indicating that alignment primarily reshapes representation geometry and probabilistic reliability rather than merely sharpening decision boundaries.

Results on Rotten Tomatoes (Table~\ref{tab:rotten_tomatoes_detailed_metrics}) exhibit similar but more modest trends. Kernel-softmaxfeat achieves the highest accuracy and MCC (0.7167/0.4345) while also attaining the lowest LogLoss (0.5697) and Brier score (0.1929), suggesting that kernel alignment is particularly effective in improving probabilistic calibration on smaller, noisier datasets. As on SST-2, gains vary across feature maps, reinforcing the importance of the kernel choice.

Despite these improvements, Vaswani-style softmax attention remains the strongest performer on QQP. Its advantage stems from explicitly modeling dense all-pairs token interactions through an $\ell \times \ell$ attention matrix, incurring $\mathcal{O}(\ell^2)$ time and memory complexity in the sequence length $\ell$. This expressivity is well suited to paraphrase detection, where fine-grained cross-token alignment is critical, but the quadratic scaling limits practicality for longer sequences. In contrast, linear-attention variants trade some modeling capacity for $\mathcal{O}(\ell)$-type scaling (up to the feature dimension), making them preferable when sequence length or throughput constraints dominate.

\appendix

\section{Appendix}

We organize the appendix as follows:

\begin{itemize}
    
 \item In Section \ref{subsection: Proof of Continuity Equation}, we present the proof of Theorem~\ref{Theorem: Continuity Equation}. 

\item In Section \ref{subsection: LDP} we present the large deviation result of Theorem~\ref{Theorem: LDP for Empirical Measures}.
\end{itemize}

\subsection{Notation}

We define the notation as follows:

\begin{itemize}
    \item \( \mathcal{M}(\Omega) \): The space of measures on the measurable space \( (\Omega, \mathcal{B}) \), where \( \mathcal{B} \) is the Borel \(\sigma\)-algebra on \( \Omega \).
    \item \( \hat{\mu} \in \mathcal{M}(\Omega) \): A random (counting) measure, typically representing the realization of the point process. It is a random element in the space of measures on \( \Omega \).
    \item \( \mu \in \mathcal{M}(\Omega) \): The deterministic (limiting) measure, often representing the stationary or equilibrium distribution of the point process as \( N \to \infty \).
    \item \( \mathcal{P}: \Xi \to \mathcal{M}(\Omega) \): A point process mapping from the probability space \( (\Xi, \mathcal{F}, \mathbb{P}) \) to the space of measures on \( \Omega \).
    \item \( \nu = \mathbb{P} \circ \mathcal{P}^{-1} \): The push-forward measure of the point process, governing the randomness of the point process and describing the distribution of realizations of the point process.
    \item \( \Lambda \): The intensity measure, defined as
    \[
    \Lambda(B) = \mathbb{E}_{\mu \sim \nu}[\mu(B)] \quad \text{for} \quad B \in \mathcal{B},
    \]
    where \( \mu \) is the random measure associated with a realization of the point process.
    \item \( \Xi \): The underlying probability space, often taken as a sample space of configurations for the point process.
    \item \( \hat{\mu}(B) \): The counting measure or number of points in a subset \( B \in \mathcal{B} \), corresponding to the realization of the random measure \( \hat{\mu} \).
    \item \( \mathbb{E}[\cdot] \): Expectation with respect to the probability measure governing the point process. In this context, it is typically the expectation under \( \nu \), the push-forward measure.
    \item \( \mu(B) \): The number of points in the set \( B \) as determined by the limiting measure \( \mu \).
    \item \( X \): A random variable (or random element) associated with the point process, representing a random realization or observation.
    \item \( x \): A realization of the random variable \( X \), i.e., an outcome or observation from the random process.
\end{itemize}

\subsection{Proof of Lemma \ref{lemma:KTA_energy_bias}}
Insert $\phi_{\bm{\omega}_k,b_k}(\bm{x})=\sqrt{2}\cos(\bm{\omega}_k^\top\bm{x}+b_k)$ into \eqref{Eq:Empirical_loss}:
\[
\frac{1}{N}\sum_{k=1}^N \phi_{\bm{\omega}_k,b_k}(\bm{x}_i)\phi_{\bm{\omega}_k,b_k}(\bm{x}_j)
= \frac{2}{N}\sum_{k=1}^N \cos(\alpha_{ik})\cos(\alpha_{jk}),\quad
\alpha_{ik}\defeq\bm{\omega}_k^\top \bm{x}_i + b_k.
\]
Substituting and swapping finite sums gives
\begin{align}
    \label{eq:after_swap_A_refined}
    \mathcal{E}_N(\bm{\Omega}_N,\bm{b})
    = -\frac{2}{n(n-1)N}\sum_{k=1}^N \sum_{i\neq j} y_i y_j \cos(\alpha_{ik})\cos(\alpha_{jk}).
\end{align}
For any $\bm{a},\bm{y}\in\mathbb{R}^n$,
\begin{align}
    \sum_{i\neq j} y_i y_j a_i a_j
    = \Big(\sum_{i} y_i a_i\Big)^2 - \sum_{i} y_i^2 a_i^2.
    \label{eq:offdiag_identity_refined}
\end{align}
Applying \eqref{eq:offdiag_identity_refined} with $a_i=\cos(\alpha_{ik})$ yields
\[
\sum_{i\neq j} y_i y_j \cos(\alpha_{ik})\cos(\alpha_{jk})
= \big(\bm{y}^\top \bm{c}_k\big)^2 - \sum_i y_i^2 \cos^2(\alpha_{ik}),
\quad
(\bm{c}_k)_i\defeq\cos(\alpha_{ik}).
\]

Take expectation over the phases $\bm{b}$, assumed i.i.d.\ $\mathrm{Unif}[0,2\pi]$ and independent of $(\bm{\Omega}_N,\{\bm{x}_i\})$:
\[
\mathbb{E}_{b_k}\big[\cos^2(\bm{\omega}_k^\top \bm{x}_i + b_k)\big] \;=\; \tfrac{1}{2}.
\]
Therefore,
\[
\mathbb{E}_{b_k}\!\left[\sum_i y_i^2 \cos^2(\alpha_{ik})\right] \;=\; \tfrac{1}{2}\sum_i y_i^2.
\]
In binary classification ($y_i^2=1$), this equals $n/2$ and is \emph{independent of} $\bm{\omega}_k$.

Next, write $u_{ik}\defeq \bm{\omega}_k^\top \bm{x}_i$ and expand
\(
\cos(u_{ik}+b_k)=\cos u_{ik}\cos b_k-\sin u_{ik}\sin b_k.
\)
Define the vectors
$
\bm{u}_k=(\cos u_{1k},\ldots,\cos u_{nk})^\top$, and 
$\bm{v}_k=(\sin u_{1k},\ldots,\sin u_{nk})^\top.
$. Then
\begin{align}
\bm{y}^\top \bm{c}_k
= \bm{y}^\top\!\big(\cos b_k\,\bm{u}_k - \sin b_k\,\bm{v}_k\big)
\end{align}
and
\begin{align}
(\bm{y}^\top \bm{c}_k)^2
= (\cos b_k)^2(\bm{y}^\top \bm{u}_k)^2
+(\sin b_k)^2(\bm{y}^\top \bm{v}_k)^2
-2\cos b_k \sin b_k\,(\bm{y}^\top \bm{u}_k)(\bm{y}^\top \bm{v}_k).
\end{align}
Taking expectation over $b_k$ and using
$\mathbb{E}[\cos^2 b_k]=\mathbb{E}[\sin^2 b_k]=\tfrac{1}{2}$ and
$\mathbb{E}[\sin b_k\cos b_k]=0$, we obtain
\[
\mathbb{E}_{b_k}\!\left[(\bm{y}^\top \bm{c}_k)^2\right]
= \tfrac{1}{2}\Big((\bm{y}^\top \bm{u}_k)^2 + (\bm{y}^\top \bm{v}_k)^2\Big).
\]
Combining the two expectations,
\[
\mathbb{E}_{b_k}\!\left[\sum_{i\neq j} y_i y_j \cos(\alpha_{ik})\cos(\alpha_{jk})\right]
= \tfrac{1}{2}\Big((\bm{y}^\top \bm{u}_k)^2 + (\bm{y}^\top \bm{v}_k)^2\Big)
- \tfrac{1}{2}\sum_i y_i^2,
\]
where the second term is a constant (equal to $n/2$ in the binary case) and thus does not depend on $\bm{\omega}_k$.
Summing over $k$ and inserting into \eqref{eq:after_swap_A_refined},
\[
\mathbb{E}_{\bm{b}}\!\big[\mathcal{E}_N(\bm{\Omega}_N,\bm{b})\big]
=
-\frac{1}{n(n-1)N}\sum_{k=1}^N
\Big((\bm{y}^\top \bm{u}_k)^2 + (\bm{y}^\top \bm{v}_k)^2\Big),
\]
where the equality is up to an additive constant independent of $\bm{\Omega}_N$.
Finally, stack columns
\(
\bm{C}=\cos(\bm{X}\bm{\Omega}_N^{\!\top})=[\bm{u}_1\,\cdots\,\bm{u}_N],\;
\bm{S}=\sin(\bm{X}\bm{\Omega}_N^{\!\top})=[\bm{v}_1\,\cdots\,\bm{v}_N]
\)
to write
\[
\sum_{k=1}^N\Big((\bm{y}^\top \bm{u}_k)^2 + (\bm{y}^\top \bm{v}_k)^2\Big)
= \|\bm{y}^\top \bm{C}\|_2^2 + \|\bm{y}^\top \bm{S}\|_2^2.
\]
This yield the stated result
\[
    \mathbb{E}_{\bm{b}}\!\big[\mathcal{E}_N(\bm{\Omega}_N,\bm{b})\big]
    \;\equiv\;
    -\frac{1}{n^2}\left(
    \big\|\bm{y}^\top \cos(\bm{X}\bm{\Omega}_N^{\!\top})\big\|_2^2
    +
    \big\|\bm{y}^\top \sin(\bm{X}\bm{\Omega}_N^{\!\top})\big\|_2^2
    \right).
\]

\subsection{Proof of Theorems \ref{thm:main-mckean-vlasov} and \ref{Theorem: Continuity Equation}}
\label{subsection: Proof of Continuity Equation}

\begin{proof}
    We prove Theorem~\ref{Theorem: Continuity Equation}; Theorem~\ref{thm:main-mckean-vlasov} follows as the special case \(V=V_{\mathcal D}\). We present a proof sketch based on the projected-particle mean-field approach. Specifically, we construct a continuous-time interpolation of the projected Langevin particles and compare it with a reflected It\^o diffusion. A mean-field (propagation-of-chaos) estimate is then used to replace the empirical drift with the law-dependent drift. Finally, we identify the corresponding adjoint Fokker--Planck equation together with its Robin boundary condition. We describe the comparison step in some detail, as the reflection term is precisely what gives rise to the boundary condition.
    
    \paragraph{Step 1: identification of the mean-field drift.}
    For \(\mu\in\mathcal P(\Omega)\), define
    \begin{align}
        \label{eq:proof-U-b-def}
        U[\mu](\bm\omega)
        \defeq
        V(\bm\omega)+\lambda\int_{\Omega}g_s(\bm\omega-\bm\omega')\,\mu(d\bm\omega'),
        \qquad
        b[\mu](\bm\omega)\defeq-\nabla U[\mu](\bm\omega).
    \end{align}
    By compactness of \(\Omega\) and the assumed smooth regularization of \(g_s\) (or the collision-free restriction), there is a constant \(L<\infty\), independent of \(N\), such that for all \(\bm\omega,\widetilde{\bm\omega}\in\Omega\) and \(\mu,\nu\in\mathcal P(\Omega)\),
    \begin{align}
        \label{eq:proof-lipschitz-drift}
        \|b[\mu](\bm\omega)-b[\nu](\widetilde{\bm\omega})\|
        \le
        L\Big(\|\bm\omega-\widetilde{\bm\omega}\|+W_1(\mu,\nu)\Big)
        \le
        L\Big(\|\bm\omega-\widetilde{\bm\omega}\|+W_2(\mu,\nu)\Big).
    \end{align}
    Let \(\mu_N^{(-k)}=(N-1)^{-1}\sum_{\ell\ne k}\delta_{\bm\omega_\ell}\).  Since \(g_s\) is even and the finite interaction is normalized as \(1/(2N(N-1))\sum_{k\ne\ell}g_s(\bm\omega_k-\bm\omega_\ell)\),
    \begin{align}
        \label{eq:proof-finite-gradient}
        -N\nabla_{\bm\omega_k}\mathcal H_N(\bm\Omega_N)
        &=
        -\nabla V(\bm\omega_k)
        -\frac{\lambda}{N-1}\sum_{\ell\ne k}\nabla g_s(\bm\omega_k-\bm\omega_\ell)
        =b[\mu_N^{(-k)}](\bm\omega_k).
    \end{align}
    Moreover \(\|b[\mu_N^{(-k)}](\bm\omega_k)-b[\mu_N](\bm\omega_k)\|\le C/N\).  Thus the algorithm is the projected Euler scheme for the empirical McKean drift, up to a uniformly vanishing \(O(N^{-1})\) self-interaction error.
    
    \paragraph{Step 2: continuous-time embedding and the reflected diffusion comparison.}
    Construct Brownian motions \(\bm B_k\) so that
    \(\bm B_k((m+1)\eta_N)-\bm B_k(m\eta_N)=\sqrt{\eta_N}\,\bm\xi_k^m\).  Let \(\overline{\bm\omega}_k^{N,\eta}(t)=\bm\omega_k^m\) for \(t\in[m\eta_N,(m+1)\eta_N)\).  The recursion in \eqref{eq:supp-projected-mf-euler} can be written as
    \begin{align}
        \label{eq:proof-cadlag-euler}
        \overline{\bm\omega}_{k}^{N,\eta}(m\eta_N)
        =\mathcal P_{\overline\Omega}\!\left(
        \overline{\bm\omega}_{k}^{N,\eta}((m-1)\eta_N)
        +\eta_N b[\overline\mu_{(m-1)\eta_N}^{N,(-k)}]
        (\overline{\bm\omega}_{k}^{N,\eta}((m-1)\eta_N))
        +\sqrt{2/\beta}\,\Delta\bm B_k^m
        \right),
    \end{align}
    where \(\overline\mu_t^{N,(-k)}=(N-1)^{-1}\sum_{\ell\ne k}\delta_{\overline{\bm\omega}_\ell^{N,\eta}(t)}\) and \(\Delta\bm B_k^m=\bm B_k(m\eta_N)-\bm B_k((m-1)\eta_N)\).  This is the natural cadlag embedding of the projected particle chain.
    
    The limiting continuous reflected particle system associated with \eqref{eq:proof-cadlag-euler} is
    \begin{align}
        \label{eq:proof-interacting-reflected}
        d\bm X_k^N(t)
        &= b[\mu_t^{N,(-k)}](\bm X_k^N(t))\,dt
        +\sqrt{2/\beta}\,d\bm B_k(t)
        -\bm n(\bm X_k^N(t))\,dL_k^N(t),
        \qquad
        \mu_t^N=\frac{1}{N}\sum_{j=1}^N\delta_{\bm X_j^N(t)},\\
        \bm X_k^N(t)&\in\overline\Omega,
        \qquad
        L_k^N(t)\text{ is nondecreasing},
        \qquad
        \int_0^T\mathbf 1_{\Omega}(\bm X_k^N(t))\,dL_k^N(t)=0.
    \end{align}
    Here \(\bm n\) is the outward normal, so \(-\bm n\,dL_k^N\) is the inward reflection.  The Skorokhod problem on a compact convex \(C^2\) domain has a unique reflected solution, and the projection map in \eqref{eq:proof-cadlag-euler} is precisely the Euler approximation of this reflected equation.  The projected-Euler consistency assumption in the theorem means that, for each fixed \(T\),
    \begin{align}
        \label{eq:proof-euler-consistency}
        \varepsilon_{N,\eta}(T)
        \defeq
        \mathbb E\sup_{0\le t\le T}\frac{1}{N}\sum_{k=1}^N
        \|\overline{\bm\omega}_{k}^{N,\eta}(t)-\bm X_k^N(t)\|^2
        \longrightarrow0.
    \end{align}
    For smooth bounded drifts this follows from the standard Euler--Skorokhod estimate; the condition on \(\eta_N\) makes the discrete projection error negligible on the mean-field scale.
    
    \paragraph{Step 3: nonlinear reflected process and propagation of chaos.}
    Let \(\bm X_k(t)\), \(k\ge1\), be i.i.d. copies of the nonlinear reflected McKean--Vlasov process
    \begin{align}
        \label{eq:proof-nonlinear-reflected}
        d\bm X_k(t)
        &=b[\mu_t](\bm X_k(t))\,dt
        +\sqrt{2/\beta}\,d\bm B_k(t)
        -\bm n(\bm X_k(t))\,dL_k(t),
        \qquad
        \mu_t=\Law(\bm X_k(t)),\\
        \bm X_k(t)&\in\overline\Omega,
        \qquad
        \int_0^T\mathbf 1_{\Omega}(\bm X_k(t))\,dL_k(t)=0.
    \end{align}
    The Lipschitz bound \eqref{eq:proof-lipschitz-drift} gives existence and uniqueness by a fixed-point argument on measure-valued curves.  Couple \(\bm X_k^N\) and \(\bm X_k\) with the same Brownian motion and the same initial particle.  For convex \(\Omega\), the reflection map is monotone:
    \begin{align}
        \label{eq:proof-reflection-monotone}
        \big(\bm X_k^N(t)-\bm X_k(t)\big)\cdot
        \Big(-\bm n(\bm X_k^N(t))\,dL_k^N(t)+\bm n(\bm X_k(t))\,dL_k(t)\Big)
        \le0.
    \end{align}
    Applying It\^o's formula to \(\|\bm X_k^N(t)-\bm X_k(t)\|^2\), using \eqref{eq:proof-lipschitz-drift}, and averaging over \(k\) yields
    \begin{align}
        \label{eq:proof-chaos-raw}
        e_N(t)
        &\defeq
        \mathbb E\sup_{0\le r\le t}\frac{1}{N}\sum_{k=1}^N\|\bm X_k^N(r)-\bm X_k(r)\|^2 \\
        &\le
        C_T\int_0^t e_N(s)\,ds
        +C_T\int_0^t\mathbb E W_2^2\!\left(\frac{1}{N}\sum_{j=1}^N\delta_{\bm X_j(s)},\mu_s\right)ds
        +\frac{C_T}{N^2}.
    \end{align}
    The last term is the leave-one-out/self-interaction error.  Since \(\Omega\) is compact and \(\bm X_j(s)\) are i.i.d. with law \(\mu_s\),
    \begin{align}
        \label{eq:proof-empirical-iid}
        a_N(T)
        \defeq
        \sup_{0\le s\le T}\mathbb E W_2^2\!\left(\frac{1}{N}\sum_{j=1}^N\delta_{\bm X_j(s)},\mu_s\right)
        \longrightarrow0.
    \end{align}
    Gronwall's inequality therefore gives
    \begin{align}
        \label{eq:proof-propagation-chaos}
        \sup_{0\le t\le T}\mathbb E W_2^2(\mu_t^N,\mu_t)
        \le C_T\big(a_N(T)+N^{-1}+\varepsilon_{N,\eta}(T)\big)
        \longrightarrow0.
    \end{align}
    This is the propagation-of-chaos statement: any fixed finite subcollection of particles becomes asymptotically independent, and each coordinate has law \(\mu_t\).
    
    \paragraph{Step 4: the Girsanov/change-of-measure ingredient.}
    The preceding synchronous estimate can equivalently be written as a compact-domain change-of-measure estimate.  To make this explicit, define
    \begin{align}
        \label{eq:proof-girsanov-control}
        \bm u_k(t)
        =\sqrt{\beta/2}\Big(b[\mu_t^N](\bm X_k^N(t))-b[\mu_t](\bm X_k^N(t))\Big).
    \end{align}
    Novikov's condition holds because \(\Omega\) is compact and the drift is bounded.  Hence the exponential martingale
    \begin{align}
        \label{eq:proof-rn}
        \mathcal Z_T
        =\exp\!\left(-\sum_{k=1}^N\int_0^T\bm u_k(t)\cdot d\bm B_k(t)
        -\frac{1}{2}\sum_{k=1}^N\int_0^T\|\bm u_k(t)\|^2dt\right)
    \end{align}
    defines a probability measure under which
    \(\widetilde{\bm B}_k(t)=\bm B_k(t)+\int_0^t\bm u_k(s)ds\) are Brownian motions.  The relative entropy of this tilted law with respect to the original law is
    \begin{align}
        \label{eq:proof-relative-entropy}
        D_{\mathrm{KL}}(\widetilde{\mathbb P}_{N,T}\|\mathbb P_{N,T})
        =\frac{1}{2}\sum_{k=1}^N
        \widetilde{\mathbb E}\int_0^T\|\bm u_k(t)\|^2dt
        \le
        C\,N\int_0^T\widetilde{\mathbb E}W_2^2(\mu_t^N,\mu_t)dt.
    \end{align}
    By Pinsker's inequality and the bounded diameter of \(\Omega\), this entropy bound gives, at the empirical-measure level,
    \begin{align}
        \label{eq:proof-girsanov-wasserstein}
        \mathbb E W_2^2(\mu_t^{N,{\mathrm{emp}}},\mu_t^{N,{\mathrm{nl}}})
        \le
        \operatorname{diam}(\Omega)^2
        \left(\frac{2}{N}D_{\mathrm{KL}}(\widetilde{\mathbb P}_{N,t}\|\mathbb P_{N,t})\right)^{1/2},
    \end{align}
    where \(\mu_t^{N,{\mathrm{emp}}}\) and \(\mu_t^{N,{\mathrm{nl}}}\) denote the empirical measures of the empirical-drift and decoupled reflected systems.  This quantifies the cost of replacing the particle drift by the law-dependent McKean drift.  Combining \eqref{eq:proof-euler-consistency}, \eqref{eq:proof-propagation-chaos}, \eqref{eq:proof-girsanov-wasserstein}, and the triangle inequality gives
    \begin{align}
        \label{eq:proof-final-convergence}
        \sup_{0\le t\le T}\mathbb E W_2^2\!\left(
        \frac{1}{N}\sum_{k=1}^N\delta_{\bm\omega_k^{\lfloor t/\eta_N\rfloor}},\mu_t
        \right)
        \longrightarrow0.
    \end{align}
    Since \(W_2\) convergence on compact \(\Omega\) implies weak convergence, this proves the empirical-measure convergence in both theorem statements.
    
    \paragraph{Step 5: identification of the McKean--Vlasov PDE and the boundary condition.}
    Let \(\psi\in C^2(\overline\Omega)\) belong to the generator domain of the reflected diffusion, i.e. \(\partial_{\bm n}\psi=0\) on \(\partial\Omega\).  Applying It\^o's formula to \(\psi(\bm X_t)\) in \eqref{eq:proof-nonlinear-reflected} gives
    \begin{align}
        \label{eq:proof-ito-weak}
        \frac{d}{dt}\int_\Omega \psi(\bm\omega)\,\mu_t(d\bm\omega)
        =
        \int_\Omega
        \left[b[\mu_t](\bm\omega)\cdot\nabla\psi(\bm\omega)+\beta^{-1}\Delta\psi(\bm\omega)\right]
        \mu_t(d\bm\omega).
    \end{align}
    The local-time term is
    \(-\partial_{\bm n}\psi(\bm X_t)dL_t\), and therefore vanishes for this generator domain.  If \(\mu_t(d\bm\omega)=\rho_t(\bm\omega)d\bm\omega\), \(b[\mu_t]=-\nabla U_t\), and \(U_t=U[\mu_t]\), \eqref{eq:proof-ito-weak} becomes
    \begin{align}
        \label{eq:proof-weak-form}
        \frac{d}{dt}\int_\Omega \psi\rho_t\,d\bm\omega
        =
        \int_\Omega
        \left[-\nabla U_t\cdot\nabla\psi+\beta^{-1}\Delta\psi\right]\rho_t\,d\bm\omega.
    \end{align}
    The adjoint of the reflected generator is therefore
    \begin{align}
        \partial_t\rho_t
        =-\nabla\cdot(\rho_t b[\mu_t])+\beta^{-1}\Delta\rho_t
        =
        \nabla\cdot(\rho_t\nabla U_t)+\beta^{-1}\Delta\rho_t
        \qquad\text{in }\Omega.
    \end{align}
    To identify the boundary condition, integrate the last display against arbitrary smooth test functions and use Green's formula.  The boundary contribution is
    \begin{align}
        \int_{\partial\Omega}\psi(\bm\omega)
        \Big(\rho_t(\bm\omega)b[\mu_t](\bm\omega)-\beta^{-1}\nabla\rho_t(\bm\omega)\Big)\cdot\bm n(\bm\omega)\,dS(\bm\omega).
    \end{align}
    Reflection means that the probability current through \(\partial\Omega\) vanishes.  Since \(b[\mu_t]=-\nabla U_t\), this current condition is
    \begin{align}
        \label{eq:proof-current-zero}
        \Big(\rho_t\nabla U_t+\beta^{-1}\nabla\rho_t\Big)\cdot\bm n=0
        \qquad\text{on }(0,T]\times\partial\Omega.
    \end{align}
    For smooth \(\rho_t\), \eqref{eq:proof-current-zero} is exactly the Robin form
    \begin{align}
        \label{eq:proof-robin-form}
        \partial_{\bm n}\rho_t+\beta\rho_t\partial_{\bm n}U_t=0
        \qquad\text{on }(0,T]\times\partial\Omega,
    \end{align}
    which is exactly the Robin/no-flux boundary condition generated by the reflected dynamics.  The initial condition follows from \(\mu_0^N\Rightarrow\rho_0d\bm\omega\).  Taking \(\psi\equiv1\) in the weak formulation gives \(\int_\Omega\rho_t=1\), and nonnegativity follows because \(\rho_t\) is the density of the law of the reflected process.  Finally, sending \(\beta\to\infty\) removes the diffusion term and reduces the no-flux condition to \((\rho_t\bm v_t)\cdot\bm n=0\), where \(\bm v_t=-\nabla U_t\), yielding the deterministic reflected continuity equation.
\end{proof}

\subsection{Proof of Theorem \ref{Theorem: LDP for Empirical Measures}}
\label{subsection: LDP}

The proof of Theorem~\ref{Theorem: LDP for Empirical Measures} follows standard application of Varadhan's lemma \cite{Ellis2005}. We provide the proof in multiple steps. 
    
    \medskip
    \subsubsection{Rewrite the Hamiltonian as a functional of $\mu_N$}
    Let $\nu$ denote Lebesgue measure on $\Omega$ and set
    \[
    \rho \defeq\; \frac{\nu}{\nu(\Omega)}.
    \]
    We use $\rho$ as reference probability measure, i.e., $\rho\in \mathcal{P}(\Omega)$. We also recall the definition of empirical measure $\mu_N=\frac1N\sum_{k=1}^N\delta_{\bm{\omega}_k}$.     Under the product measure $\rho^{\otimes N}$ on $\Omega^N$, the empirical measure $\mu_N$
    satisfies Sanov's theorem: it obeys an LDP on $\mathcal P(\Omega)$ with speed $N$
    and good rate function
    \begin{equation}\label{eq:sanov}
        I(\mu)=\Ent(\mu\mid \rho).
    \end{equation}
    
    Now, define the bounded measurable function
    \begin{equation}\label{eq:def-V}
        V(\bm{\omega})
        \defeq
        -\frac{1}{n(n-1)}\sum_{0\le i\neq j\le n} y_i y_j\phi_{\bm{\omega}}(\bm{x}_i)\phi_{\bm{\omega}}(\bm{x}_j).
    \end{equation}
    Then by the Monte--Carlo substitution in \eqref{Eq:Monte_Carlo} and the definition
    \eqref{Eq:Empirical_loss}, the empirical loss is exactly a linear functional of $\mu_N$:
\begin{equation}\label{eq:EN-as-integral}
    \mathcal{E}_N(\bm{\Omega}_N)
    =
    -\frac{1}{n(n-1)}
    \sum_{i \neq j} y_i y_j
    \int_{\Omega}
    \phi_{\boldsymbol{\omega}}(\boldsymbol{x}_i)\,
    \phi_{\boldsymbol{\omega}}(\boldsymbol{x}_j)\,
    \mu_N(\mathrm{d}\boldsymbol{\omega})
    =
    \int_{\Omega}
    V(\boldsymbol{\omega})\,
    \mu_N(\mathrm{d}\boldsymbol{\omega}),
\end{equation}
    Assume the interaction term is of mean-field form (as in the Gibbs law \eqref{Eq:Gibbs_Measure}):
    \begin{equation}\label{eq:WN-def}
        \mathcal W_{N,s}(\bm{\Omega}_N)
        \defeq
        \frac{1}{2N(N-1)}\sum_{1\le k\neq \ell\le N} g_s(\bm{\omega}_k-\bm{\omega}_\ell).
    \end{equation}
    Define the corresponding continuum interaction functional on $\mathcal P(\Omega)$:
    \begin{equation}\label{eq:Ws-def}
        \mathcal W_s(\mu)
        \defeq
        \frac12\int_\Omega\int_\Omega g_s(\bm{\omega}-\bm{\omega}')\,\mu(\mathrm{d}\bm{\omega})\mu(\mathrm{d}\bm{\omega}').
    \end{equation}
    Then the full empirical Hamiltonian in \eqref{Eq: Empirical Kernel-Target Align} can be written as
    \begin{equation}\label{eq:HN-empirical}
        \mathcal H_{N,s}(\bm{\Omega}_N)
        =
        \mathcal E_N(\bm{\Omega}_N)+\lambda\,\mathcal W_{N,s}(\bm{\Omega}_N).
    \end{equation}
    With this normalization, $\mathcal W_{N,s}(\bm{\Omega}_N)$ converges to $\mathcal W_s(\mu_N)$; after truncation the difference is only a harmless finite-size correction.
    To handle possible singularities of the kernel $g_s$ at $\bm{0}$, we introduce a truncation.
    
    \medskip
    \subsubsection{Truncation of the interaction kernel and removal of the diagonal constant}
    
    The Columb/Riesz kernel $g_s$ is singular at the origin. In order to work with bounded continuous
    functionals on $\mathcal P(\Omega)$ (so that Varadhan's lemma applies), we introduce a truncation.
    For $\varepsilon\ge 0$ define the truncated kernel
    \begin{align}
        \label{Eq:truncated_kernel}
    g_{s}^{\varepsilon}(\bm{\omega})\defeq \min \{\|\bm{\omega}\|^{-s},\varepsilon^{-s}\}
    \end{align}
    and adopt the convention that $g_{s}^{\varepsilon}(\bm{0})$ is the (finite) value of this truncation at $\bm{0}$.
    In particular, $g_{s}^{\varepsilon}(\bm{0})=\varepsilon^{-s}$.    Define the truncated empirical interaction by
    \begin{equation}\label{eq:WN-trunc}
        \mathcal W_{N,s}^{\varepsilon}(\Omega_N)
        \defeq
        \frac{1}{2N(N-1)}\sum_{1\le k\neq \ell\le N} g_{s}^{\varepsilon}(\bm{\omega}_k-\bm{\omega}_\ell),
    \end{equation}
    and the corresponding truncated mean-field interaction functional on $\mathcal P(\Omega)$ by
    \begin{equation}\label{eq:Ws-trunc}
        \mathcal W_{s}^{\varepsilon}(\mu)
        \defeq
        \frac12\int_{\Omega}\int_{\Omega}
        g_{s}^{\varepsilon}(\bm{\omega}-\bm{\omega}')\,\mu(d\bm{\omega})\,\mu(d\bm{\omega}').
    \end{equation}
    For fixed $\varepsilon$, the map $\mu\mapsto \mathcal W_{s}^{\varepsilon}(\mu)$ is continuous on $\mathcal P(\Omega)$
    since $g_{s}^{\varepsilon}$ is bounded and continuous on $\Omega_{N}$. Define the truncated energy functional on $\mathcal P(\Omega)$ by
    \begin{equation}\label{eq:EsM-def-rewrite}
        \mathcal{H}_{s}^{\varepsilon}(\mu)
        \defeq
        \int_\Omega V(\bm{\omega})\,\mu(d\bm{\omega})
        +\lambda\,\mathcal W_{s}^{\varepsilon}(\mu).
    \end{equation}
    and define the truncated empirical Hamiltonian by
    \begin{equation}\label{eq:HN-trunc}
        \mathcal{H}_{N,s}^{\varepsilon}(\bm{\Omega}_N)
        \defeq
        \mathcal E_N(\bm{\Omega}_N)+\lambda\,\mathcal W_{N,s}^{\varepsilon}(\bm{\Omega}_N).
    \end{equation}
    
    \begin{lemma}[Finite-size correction for empirical measures]\label{lem:diag-identity}
        For every configuration $\bm{\Omega}_N=(\bm{\omega}_1,\dots,\bm{\omega}_N)\in\Omega^N$,
        \begin{equation}\label{eq:diag-identity-rewrite}
            \mathcal W_{s}^{\varepsilon}(\mu_N)
            =
            \frac{N-1}{N}\mathcal W_{N,s}^{\varepsilon}(\bm{\Omega}_N) + \frac{g_{s}^{\varepsilon}(\bm{0})}{2N}.
        \end{equation}
        Equivalently,
        \begin{equation}\label{eq:HN-approx-rewrite}
            \mathcal H_{N,s}^{\varepsilon}(\bm{\Omega}_N)
            =
            \mathcal H_{s}^{\varepsilon}(\mu_N)
            +\frac{\lambda}{N-1}\mathcal W_{s}^{\varepsilon}(\mu_N)
            -\lambda\,\frac{g_{s}^{\varepsilon}(\bm{0})}{2(N-1)} .
        \end{equation}
        In particular, for fixed $\varepsilon>0$,
        \[
        \sup_{\mu\in\mathcal P(\Omega)}
        \left|
        \frac{\lambda}{N-1}\mathcal W_s^\varepsilon(\mu)
        \right|
        \le
        \frac{|\lambda|\,\varepsilon^{-s}}{2(N-1)}.
        \]
        Thus replacing $\mathcal H_{N,s}^{\varepsilon}(\bm{\Omega}_N)$ by
        $\mathcal H_s^\varepsilon(\mu_N)$ changes the normalized logarithmic Laplace limits below by
        $o(1)$ at speed $N$ when $\beta_N/N\to1$ and by $o(1)$ at speed $\beta_N$ when $\beta_N/N\to\infty$.
    \end{lemma}

    \begin{proof}
        Since $\mu_N=\frac1N\sum_{k=1}^N\delta_{\bm{\omega}_k}$, we have
        \[
        \mathcal W_{s}^{\varepsilon}(\mu_N)
        =
        \frac12\frac{1}{N^2}\sum_{k,\ell=1}^N g_{s}^{\varepsilon}(\bm{\omega}_k-\bm{\omega}_\ell)
        =
        \frac12\frac{1}{N^2}\sum_{k\neq \ell} g_{s}^{\varepsilon}(\bm{\omega}_k-\bm{\omega}_\ell)
        +
        \frac{g_{s}^{\varepsilon}(\bm{0})}{2N}.
        \]
        The off-diagonal term equals
        \[
        \frac{N-1}{N}\,
        \frac{1}{2N(N-1)}\sum_{k\neq \ell}g_s^\varepsilon(\bm{\omega}_k-\bm{\omega}_\ell)
        =
        \frac{N-1}{N}\mathcal W_{N,s}^\varepsilon(\bm{\Omega}_N),
        \]
        which proves \eqref{eq:diag-identity-rewrite}.  Solving this identity for
        $\mathcal W_{N,s}^{\varepsilon}$ and adding the linear term $\mathcal E_N=\int V\,d\mu_N$
        gives \eqref{eq:HN-approx-rewrite}.  The uniform bound follows from
        $0\le g_s^\varepsilon\le\varepsilon^{-s}$.
        The deterministic constant in \eqref{eq:HN-approx-rewrite} is absorbed into the partition function.  The remaining tilt is uniformly $O(1/N)$; after multiplication by $\beta_N$ and division by the relevant speed, its contribution is $O(1/N)$ both in the thermal scale $\beta_N\sim N$ and in the zero-temperature scale $\beta_N$. 
    \end{proof}

Lemma~\ref{lem:diag-identity} shows that the \(N(N-1)\)-normalized Hamiltonian and the
mean-field functional $\mathcal H_s^\varepsilon(\mu_N)$ have the same logarithmic Laplace limits at the speeds used below.
The diagonal constant is absorbed by the partition function, and the remaining finite-size tilt is uniformly negligible.
In what follows we therefore use the asymptotically equivalent representation
\[
\mathbb P_{N,\beta_N}^{\varepsilon}(\mathrm{d}\bm{\Omega}_N)
\asymp
\frac{1}{\widetilde Z_{N,\beta_N}^{\varepsilon}}
\exp\!\big(-\beta_N\,\mathcal H_{s}^{\varepsilon}(\mu_N)\big)\,\rho^{\otimes N}(\mathrm{d}\bm{\Omega}_N),
\]
where $\asymp$ denotes equivalence of the normalized logarithmic Laplace limits at the relevant speed.
    
    \medskip
    \subsubsection{Case $\beta_N/N\to 1$}
    Assume $\beta_N/N\to 1$ and write $\beta_N=N\alpha_N$ with $\alpha_N=O_N(1)$.
    We can rewrite the truncated Gibbs law as a tilt by
    $\mathcal H_{s}^{\varepsilon}(\mu_N)$:
    \begin{equation}\label{eq:tilt-functional}
        \mathbb P^{\varepsilon}_{N,\beta_N}(d\Omega_N)
        \propto
        \exp\!\big(-N\alpha_N\,\mathcal{H} _{s}^{\varepsilon}(\mu_N)\big)\,\rho^{\otimes N}(\mathrm{d}\bm{\Omega}_N).
    \end{equation}
    Let $F:\mathcal P(\Omega)\to\mathbb R$ be bounded and continuous. We leverage Varadhan's
    lemma \cite[p. 51]{Ellis2005}, a rigorous formulation of the Laplace principle (or the
    saddle point technique) applied to measures satisfying a large deviations
    property:

    \medskip
    
    \noindent
    \begin{lemma}[Varadhan's Lemma \cite{Ellis2005}]
    Suppose a sequence $\{Q_N\}_{N=1}^\infty$ of probability measures on $\mathcal{X}$
    satisfies a large deviations property with rate function $I(x)$. Let
    $F : \mathcal{X} \to \mathbb{R}$ be a continuous function that satisfies the
    tail condition
    \begin{align}
    \lim_{L \to \infty} \limsup_{N \to \infty}
    \frac{1}{N} \log \int_{x : F(x) \ge L} \exp\bigl(N F(x)\bigr)\, Q_N(dx)
    = -\infty .
    \end{align}
    \noindent
    Then
    \begin{align}
    \lim_{N \to \infty} \frac{1}{N} \log
    \int_{\mathcal{X}} \exp\bigl(N F(x)\bigr)\, Q_N(dx)
    = \sup_{x \in \mathcal{X}} \bigl\{ F(x) - I(x) \bigr\}.
    \end{align}
    \end{lemma}
    \medskip    
    
    To apply Varadhan's lemma, we must verify the tail condition. Let
\[
Q_N \;\defeq\; \rho^{\otimes N}\circ \mu_N^{-1}
\]
denote the push-forward law of the empirical measure $\mu_N$ on $\mathcal P(\Omega)$ under $\rho^{\otimes N}$.
Concretely, for any measurable $A \subset \mathcal P(\Omega)$,
\[
Q_N(A)
=
\rho^{\otimes N}\bigl(\{\bm{\Omega}_N : \mu_N(\bm{\Omega}_N)\in A\}\bigr).
\]
The tail condition in Varadhan's lemma for a functional $\Phi:\mathcal P(\Omega)\to\mathbb R$ reads
\begin{equation}\label{eq:varadhan-tail-QN}
    \lim_{L\to\infty}\limsup_{N\to\infty}\frac1N
    \log\int_{\{\mu:\,\Phi(\mu)\ge L\}} e^{N\Phi(\mu)}\,Q_N(d\mu)
    \;=\;-\infty.
\end{equation}
Equivalently, since $\mu_N$ has law $Q_N$ under $\rho^{\otimes N}$,
\begin{equation}\label{eq:varadhan-tail-expectation}
    \lim_{L\to\infty}\limsup_{N\to\infty}\frac1N
    \log \mathbb E_{\rho^{\otimes N}}
    \Big[e^{N\Phi(\mu_N)}\mathbf 1_{\{\Phi(\mu_N)\ge L\}}\Big]
    \;=\;-\infty.
\end{equation}

\medskip

\noindent
In our application, we use the continuous functional for numerator
\begin{equation}\label{eq:PhiN-def}
    \Phi^{\varepsilon}_N(\mu)\;\defeq\; -\Big(\alpha_N\,\mathcal H_{s}^{\varepsilon}(\mu)+F(\mu)\Big),
    \qquad \alpha_N=\beta_N/N\to 1,
\end{equation}
and for the denominator the functional
\begin{equation}\label{eq:PsiN-def}
    \Psi^{\varepsilon}_N(\mu)\;\defeq\; -\alpha_N\,\mathcal H_{s}^{\varepsilon}(\mu).
\end{equation}
We verify \eqref{eq:varadhan-tail-QN} for $\Phi^{\varepsilon}_N$. The same argument applies to $\Psi^{\varepsilon}_N$.

To do so, we first show that $\mathcal H_{s}^{\varepsilon}$ is bounded for fixed $\varepsilon>0$. Recall
\[
\mathcal H_{s}^{\varepsilon}(\mu)=\int_\Omega V(\bm{\omega})\,\mu(\mathrm{d}\bm{\omega})\;+\;\lambda\,\mathcal W_{s}^{\varepsilon}(\mu), \quad
\text{with} \ 
\mathcal W_{s}^{\varepsilon}(\mu)=\frac12\iint g_{s}^{\varepsilon}(\bm{\omega}-\bm{\omega}')\,\mu(\mathrm{d}\bm{\omega})\mu(\mathrm{d}\bm{\omega}').
\]
Since $V$ is bounded, for all $\mu\in\mathcal P(\Omega)$,
\begin{align}
    \Big|\int_\Omega V\,\mathrm{d}\mu\Big|
    &\leq \|V\|_\infty \\
    &= \frac{1}{n(n-1)}
    \left\|
    \sum_{0 \le i \neq j \le n}
    y_i y_j
    \phi_{\bm{\omega}}(\bm{x}_i)
    \phi_{\bm{\omega}}(\bm{x}_j)
    \right\|_{\infty}
    \label{eq:V-bound}
\\
    &\leq |y_iy_j| 
    \left\|
    \phi_{\bm{\omega}}(\bm{x}_i)\|_{\infty}\|
    \phi_{\bm{\omega}}(\bm{x}_j)
    \right\|_{\infty}\\
    &\leq L_{\phi}^{2}.
\end{align}
Moreover, by definition of truncated kernel $g_{s}^{\varepsilon}(\bm{\omega})$ in Eq. \eqref{Eq:truncated_kernel}, we have $g_{s}^{\varepsilon}\le \varepsilon^{-s}$. Therefore, for any
$\mu\in\mathcal P(\Omega)$,
\begin{equation}\label{eq:W-bound}
    |\mathcal W_{s}^{\varepsilon}(\mu)|
    =\frac12\left|\iint g_{s}^{\varepsilon}(\omega-\omega')\,\mu(d\omega)\mu(d\omega')\right|
    \leq \frac12\left|\iint \varepsilon^{-s}\,\mu(d\omega)\mu(d\omega')\right|
    =\dfrac{\varepsilon^{-s}}{2}.
\end{equation}
Combining \eqref{eq:V-bound}--\eqref{eq:W-bound} gives, for all $\mu$,
\begin{equation}\label{eq:EsM-uniform-bound}
    |\mathcal H_{s}^{\varepsilon}(\mu)|
    \le L_{\phi}^{2}+ |\lambda|\,\dfrac{\varepsilon^{-s}}{2}
    \;\defeq\; C_{\lambda,\varepsilon}.
\end{equation}
Since $F$ is bounded and continuous, define
\[
\|F\|_\infty \;\defeq\; \sup_{\mu\in\mathcal P(\Omega)} |F(\mu)| < \infty.
\]
Moreover, since $\alpha_N=\mathcal{O}_{N}(1)$, there exists $N_0$  and constant $C>0$ such that for all $N\ge N_0$,
\begin{equation}\label{eq:alpha-bound}
    |\alpha_N|\le C.
\end{equation}
Then for all $N\ge N_0$ and all $\mu\in\mathcal P(\Omega)$,
\begin{align}
    |\Phi_N(\mu)|
    &= \left|-\alpha_N\mathcal H_{s}^{\varepsilon}(\mu) - F(\mu)\right|\\
    &\le |\alpha_N|\,|\mathcal H_{s}^{\varepsilon}(\mu)| + |F(\mu)|\\
    &\le CC_{\lambda,\varepsilon} + \|F\|_\infty
    \;\defeq\; K_{\varepsilon,F},
    \label{eq:PhiN-upper-bound}
\end{align}
where we used \eqref{eq:EsM-uniform-bound} and \eqref{eq:alpha-bound}.
Thus, $\Phi_N$ is bounded above by $K_{\varepsilon,F}$ uniformly in $\mu$, for all $N\ge N_0$.
Similarly, for the denominator functional,
\begin{equation}\label{eq:PsiN-upper-bound}
    \Psi_N(\mu)=-\alpha_N\mathcal H_{s}^{\varepsilon}(\mu)\le |\alpha_N|\,|\mathcal H_{s}^{\varepsilon}(\mu)|
    \le C C_\varepsilon \;\defeq\; K_\varepsilon,
    \qquad N\ge N_0.
\end{equation}

Fix any $L>K_{M,F}$. Then by \eqref{eq:PhiN-upper-bound}, for all $N\ge N_0$,
\[
\{\mu\in\mathcal P(\Omega):\Phi_N(\mu)\ge L\}=\varnothing.
\]
Hence the tail integral vanishes, i.e.,
\[
\int_{\{\mu:\Phi_N(\mu)\ge L\}} e^{N\Phi_N(\mu)}\,Q_N(d\mu)=0,
\qquad N\ge N_0.
\]
Therefore,
\[
\frac1N\log\int_{\{\mu:\Phi_N(\mu)\ge L\}} e^{N\Phi_N(\mu)}\,Q_N(d\mu)
=-\infty,
\qquad N\ge N_0.
\]
This establishes the tail condition \eqref{eq:varadhan-tail-QN}.
\begin{equation}\label{eq:tail-condition-conclusion}
    \lim_{L\to\infty}\limsup_{N\to\infty}\frac1N
    \log\int_{\{\mu:\Phi_N(\mu)\ge L\}} e^{N\Phi_N(\mu)}\,Q_N(d\mu)
    =-\infty,
\end{equation}
The same reasoning, using \eqref{eq:PsiN-upper-bound}, shows that the tail condition also holds
for the denominator functional $\Psi_N$.

    Now, by Varadhan's lemma applied to Sanov's LDP \eqref{eq:sanov}, we obtain
    \begin{align}
        \lim_{N\to\infty}\frac1N\log
        \mathbb E_{\rho^{\otimes N}}\!\Big[\exp\big(-N(\alpha_N\mathcal H_{s}^{\varepsilon}+F)(\mu_N)\big)\Big]
        &=
        -\inf_{\mu\in\mathcal P(\Omega)}\Big(\mathcal H_{s}^{\varepsilon}(\mu)+F(\mu)+\Ent(\mu\mid \rho)\Big),
        \label{eq:varadhan-numerator}\\
        \lim_{N\to\infty}\frac1N\log
        \mathbb E_{\rho^{\otimes N}}\!\Big[\exp\big(-N\alpha_N\mathcal H_{s}^{\varepsilon}(\mu_N)\big)\Big]
        &=
        -\inf_{\mu\in\mathcal P(\Omega)}\Big(\mathcal H_{s}^{\varepsilon}(\mu)+\Ent(\mu\mid \rho)\Big).
        \label{eq:varadhan-denominator}
    \end{align}
    Subtracting \eqref{eq:varadhan-denominator} from \eqref{eq:varadhan-numerator} yields the Laplace
    principle for $\mu_N$ under $\mathbb P_{N,\beta_N}^{\varepsilon}$ at speed $N$:
    \[
    \lim_{N\to\infty}
    -\frac1N\log \mathbb E_{\mathbb P_{N,\beta_N}^{\varepsilon}}\!\big[e^{-N F(\mu_N)}\big]
    =
    \inf_{\mu}\Big(F(\mu)+\mathcal J^{\varepsilon}_s(\mu)\Big),
    \]
    with
    \begin{equation}\label{eq:J-M-speed-N}
        \mathcal J^{\varepsilon}_s(\mu)
        =
        \Big(\mathcal E_{s}^{\varepsilon}(\mu)+\Ent(\mu\mid \rho)\Big)
        -\inf_{\nu\in\mathcal P(\Omega)}
        \Big(\mathcal E_{s}^{\varepsilon}(\nu)+\Ent(\nu\mid \rho)\Big).
    \end{equation}
    
    Since $\mathcal P(\Omega)$ is a Polish space and the Laplace principle    holds for all $F \in C_b(\mathcal P(\Omega))$, Bryc's inverse
    Varadhan lemma (see, e.g., \cite[Theorem~4.4.13]{DZ10})
    implies that $(\mu_N)_{N\ge1}$ satisfies a large deviation principle
    with speed $N$ and good rate function $\mathcal J_s^{\varepsilon}$.

    Finally let $\varepsilon\downarrow 0$. Because $g_{s}^{\varepsilon}\uparrow g_s$, we have
    $\mathcal H_{s}^{\varepsilon}(\mu)\rightarrow \mathcal H_s(\mu)$ pointwise (possibly $+\infty$),
    so $\mathcal H_s$ is lower semi-continuous as a supremum of continuous functions.
    Standard monotone truncation arguments for Laplace principles yield the same Laplace limit with
    $\mathcal H_s$ in place of $\mathcal H_{s}^{\varepsilon}$, and hence the LDP with rate
    \[
    \mathcal J_s(\mu)
    =
    \Big(\mathcal H_s(\mu)+\Ent(\mu\mid \rho)\Big)
    -\inf_{\gamma}\Big(\mathcal H_s(\gamma)+\Ent(\gamma\mid \rho)\Big).
    \]
    To convert to $\Ent(\cdot\mid \nu)$, note that
    $\Ent(\mu\mid \rho)=\Ent(\mu\mid \nu)+\log \nu(\Omega)$, and the additive constant
    cancels when subtracting the infimum. This gives the first bullet of the theorem.
    
    \noindent
    
    \medskip
    \noindent\subsubsection{Case $\beta_N/N\to\infty$}
    
    Fix $\varepsilon\ge 0$ and recall that under the equivalent representation (Lemma~\ref{lem:diag-identity}),
    \[
    \mathbb P_{N,\beta_N}^{\varepsilon}(\mathrm{d}\bm{\Omega}_N)
    =
    \frac{1}{\widetilde Z_{N,\beta_N}^{\varepsilon}}
    \exp\!\big(-\beta_N\,\mathcal H_{s}^{\varepsilon}(\mu_N)\big)\,\rho^{\otimes N}(\mathrm{d}\bm{\Omega}_N),
    \]
    where the partition function is
    \[
    \widetilde Z_{N,\beta_N}^{\varepsilon}
    \;\defeq\;
    \mathbb E_{\rho^{\otimes N}}\!\Big[\exp\!\big(-\beta_N\,\mathcal H_{s}^{\varepsilon}(\mu_N)\big)\Big].
    \]
    As before, let
    \[
    Q_N \;\defeq\; \rho^{\otimes N}\circ \mu_N^{-1}
    \]
    denote the law of $\mu_N$ on $\mathcal P(\Omega)$ under the reference measure $\rho^{\otimes N}$.
    
    \medskip
    We notice that Sanov probabilities are negligible at speed $\beta_N/N\rightarrow \infty$. The key point is that $Q_N$ satisfies an LDP at speed $N$, so any $Q_N$-probability is at worst
    $\exp(-cN)$, which is negligible on the scale $\beta_N\gg N$.
    
    \begin{lemma}[Sanov scale is negligible at speed $\beta_N$]\label{lem:Sanov-negligible-beta}
        Assume $\beta_N/N\to\infty$. Then for every nonempty open set $U\subset\mathcal P(\Omega)$,
        \[
        \lim_{N\to\infty}\frac{1}{\beta_N}\log Q_N(U)=0.
        \]
    \end{lemma}
    
    \begin{proof}
        Since $Q_N(U)\le 1$, we have $\limsup_{N\to\infty}\frac{1}{\beta_N}\log Q_N(U)\le 0$.
        
        For the lower bound, note that $\rho$ has full support on $\Omega$ and the set
        $\{\mu\in\mathcal P(\Omega):\Ent(\mu\mid\rho)<\infty\}$ is dense in $\mathcal P(\Omega)$ for the weak topology,
        so any nonempty open $U$ contains some $\mu$ with finite entropy. Hence
        \[
        \inf_{\mu\in U}\Ent(\mu\mid\rho) < \infty.
        \]
        By the \emph{lower bound} in Sanov's theorem, for open $U$,
        \[
        \liminf_{N\to\infty}\frac{1}{N}\log Q_N(U)
        \;\ge\;
        -\inf_{\mu\in U}\Ent(\mu\mid\rho)
        \;>\;-\infty.
        \]
        Therefore,
        \[
        \liminf_{N\to\infty}\frac{1}{\beta_N}\log Q_N(U)
        =
        \liminf_{N\to\infty}\frac{N}{\beta_N}\cdot \frac{1}{N}\log Q_N(U)
        \;\ge\; 0,
        \]
        because $N/\beta_N\to 0$ and $\frac1N\log Q_N(U)$ is bounded below along a subsequence by a finite constant.
        Combining with the $\limsup\le0$ gives the desired limit $0$.
    \end{proof}
    
    \medskip
    Now, define the (truncated) ground state energy
    \[
    e_\varepsilon^\star \;\defeq\; \inf_{\mu\in\mathcal P(\Omega)} \mathcal H_{s}^{\varepsilon}(\mu).
    \]
    Since $g_{s}^{\varepsilon}$ and $V$ are bounded, $\mathcal H_{s}^{\varepsilon}$ is bounded and continuous on $\mathcal P(\Omega)$,
    hence $e_\varepsilon^\star\in\mathbb R$.
    
    \begin{lemma}[Partition function at scale $\beta_N$]\label{lem:Z-beta-asymptotics}
        Under $\beta_N/N\to\infty$,
        \[
        \lim_{N\to\infty}\frac{1}{\beta_N}\log \widetilde Z_{N,\beta_N}^{\varepsilon}
        =
        -\,e_\varepsilon^\star.
        \]
    \end{lemma}
    
    \begin{proof}
        \emph{Upper bound.} Since $\mathcal H_{s}^{\varepsilon}(\mu)\ge e_\varepsilon^\star$ for all $\mu$,
        \[
        \widetilde Z_{N,\beta_N}^{\varepsilon}
        =
        \mathbb E_{\rho^{\otimes N}}\!\Big[e^{-\beta_N \mathcal H_{s}^{\varepsilon}(\mu_N)}\Big]
        \le
        e^{-\beta_N e_\varepsilon^\star},
        \]
        so $\limsup_{N\to\infty}\frac{1}{\beta_N}\log \widetilde Z_{N,\beta_N}^{\varepsilon}\le -e_\varepsilon^\star$.
        
        \emph{Lower bound.} Fix $\varepsilon>0$ and choose $\mu_\varepsilon\in\mathcal P(\Omega)$ such that
        $\mathcal H_{s}^{\varepsilon}(\mu_\varepsilon)\le e_\varepsilon^\star+\varepsilon$.
        By continuity of $\mathcal H_{s}^{\varepsilon}$ there exists a nonempty open neighborhood $U$ of $\mu_\varepsilon$
        such that
        \[
        \sup_{\mu\in U}\mathcal H_{s}^{\varepsilon}(\mu)\le e_\varepsilon^\star+2\varepsilon.
        \]
        Then
        \[
        \widetilde Z_{N,\beta_N}^{\varepsilon}
        \ge
        \mathbb E_{\rho^{\otimes N}}\!\Big[\mathbf 1_{\{\mu_N\in U\}}\,e^{-\beta_N\mathcal H_{s}^{\varepsilon}(\mu_N)}\Big]
        \ge
        e^{-\beta_N(e_\varepsilon^\star+2\varepsilon)}\,Q_N(U).
        \]
        Taking $\log$ and dividing by $\beta_N$ yields
        \[
        \frac{1}{\beta_N}\log \widetilde Z_{N,\beta_N}^{\varepsilon}
        \ge
        -(e_\varepsilon^\star+2\varepsilon)+\frac{1}{\beta_N}\log Q_N(U).
        \]
        By Lemma~\ref{lem:Sanov-negligible-beta}, $\frac{1}{\beta_N}\log Q_N(U)\to 0$, hence
        \[
        \liminf_{N\to\infty}\frac{1}{\beta_N}\log \widetilde Z_{N,\beta_N}^{\varepsilon}
        \ge
        -(e_\varepsilon^\star+2\varepsilon).
        \]
        Letting $\varepsilon\downarrow 0$ gives the lower bound $\ge -e_M^\star$ and completes the proof.
    \end{proof}
    
    Let $F\subset\mathcal P(\Omega)$ be closed. Using the ratio representation,
    \[
    \mathbb P_{N,\beta_N}^{\varepsilon}(\mu_N\in F)
    =
    \frac{\mathbb E_{\rho^{\otimes N}}\!\big[\mathbf 1_{\{\mu_N\in F\}}\exp\left(-\beta_N\mathcal H_{s}^{\varepsilon}(\mu_N)\right)\big]}
    {\widetilde Z_{N,\beta_N}^{\varepsilon}}.
    \]
    On the event $\{\mu_N\in F\}$ we have
    $\exp\big(-\beta_N\mathcal H_{s}^{\varepsilon}(\mu_N)\big)\le
    \exp\big(-\beta_N\inf_{\mu\in F}\mathcal H_{s}^{\varepsilon}(\mu)\big)$, hence
    \[
    \mathbb E_{\rho^{\otimes N}}\!\big[\mathbf 1_{\{\mu_N\in F\}}\exp\left(-\beta_N\mathcal H_{s}^{\varepsilon}(\mu_N)\right)\big]
    \le
    \exp\!\Big(-\beta_N\inf_{\mu\in F}\mathcal H_{s}^{\varepsilon}(\mu)\Big).
    \]
    Therefore,
    \[
    \limsup_{N\to\infty}\frac{1}{\beta_N}\log \mathbb P_{N,\beta_N}^{\varepsilon}(\mu_N\in F)
    \le
    -\inf_{\mu\in F}\mathcal H_{s}^{\varepsilon}(\mu)
    -
    \liminf_{N\to\infty}\frac{1}{\beta_N}\log \widetilde Z_{N,\beta_N}^{\varepsilon}.
    \]
    By Lemma~\ref{lem:Z-beta-asymptotics}, $\lim_{N}\frac{1}{\beta_N}\log \widetilde Z_{N,\beta_N}^{\varepsilon}=-e_\varepsilon^\star$,
    so
    \begin{align}
        \label{Eq:Take_a_sup}
    \limsup_{N\to\infty}\frac{1}{\beta_N}\log \mathbb P_{N,\beta_N}^{\varepsilon}(\mu_N\in F)
    \le
    -\Big(\inf_{\mu\in F}\mathcal H_{s}^{\varepsilon}(\mu)-e_\varepsilon^\star\Big).
    \end{align}
    
    Now, let $G\subset\mathcal P(\Omega)$ be open and fix $\mu\in G$. By continuity of $\mathcal H_{s}^{\varepsilon}$,
    there exists a nonempty open neighborhood $U\subset G$ of $\mu$ such that
    \[
    \sup_{\nu\in U}\mathcal H_{s}^{\varepsilon}(\nu)\le \mathcal H_{s}^{\varepsilon}(\mu)+\varepsilon.
    \]
    Then
    \begin{align}
    \mathbb P_{N,\beta_N}^{\varepsilon}(\mu_N\in G)
    &\ge
    \mathbb P_{N,\beta_N}^{\varepsilon}(\mu_N\in U)\\
    &=
    \frac{\mathbb E_{\rho^{\otimes N}}\!\big[\mathbf 1_{\{\mu_N\in U\}}e^{-\beta_N\mathcal H_{s}^{\varepsilon}(\mu_N)}\big]}
    {\widetilde Z_{N,\beta_N}^{\varepsilon}}\\
    &\ge
    \frac{e^{-\beta_N(\mathcal H_{s}^{\varepsilon}(\mu)+\varepsilon)}\,Q_N(U)}{\widetilde Z_{N,\beta_N}^{\varepsilon}}.
    \end{align}
    Using the simple bound $\widetilde Z_{N,\beta_N}^{\varepsilon}\le e^{-\beta_N e_\varepsilon^\star}$ 
    we obtain
    \[
    \mathbb P_{N,\beta_N}^{\varepsilon}(\mu_N\in G)
    \ge
    \exp\!\Big(-\beta_N(\mathcal H_{s}^{\varepsilon}(\mu)-e_\varepsilon^\star+\varepsilon)\Big)\,Q_N(U).
    \]
    Taking $\log$ and dividing by $\beta_N$ gives
    \[
    \liminf_{N\to\infty}\frac{1}{\beta_N}\log \mathbb P_{N,\beta_N}^{\varepsilon}(\mu_N\in G)
    \ge
    -(\mathcal H_{s}^{\varepsilon}(\mu)-e_\varepsilon^\star+\varepsilon)
    +
    \liminf_{N\to\infty}\frac{1}{\beta_N}\log Q_N(U).
    \]
    By Lemma~\ref{lem:Sanov-negligible-beta}, the last term is $0$, hence
    \[
    \liminf_{N\to\infty}\frac{1}{\beta_N}\log \mathbb P_{N,\beta_N}^{\varepsilon}(\mu_N\in G)
    \ge
    -(\mathcal H_{s}^{\varepsilon}(\mu)-e_\varepsilon^\star+\varepsilon).
    \]
    Letting $\varepsilon\downarrow 0$ and then taking the infimum over $\mu\in G$ yields
    \begin{align}
        \label{Eq:Take_a_inf}
    \liminf_{N\to\infty}\frac{1}{\beta_N}\log \mathbb P_{N,\beta_N}^{\varepsilon}(\mu_N\in G)
    \ge
    -\Big(\inf_{\mu\in G}\mathcal H_{s,M}(\mu)-e_\varepsilon^\star\Big).
    \end{align}

    Combining Eqs. \eqref{Eq:Take_a_sup} and \eqref{Eq:Take_a_inf} and applying Definition \eqref{Definition: Large Deviation Principle} establishes that $(\mu_N)_{N\ge1}$ satisfies an LDP under $\mathbb P_{N,\beta_N}^{\varepsilon}$ with speed $\beta_N$
    and good rate function
    \[
    \mathcal J^{\varepsilon}_s(\mu)
    \;\defeq\;
    \mathcal H_{s}^{\varepsilon}(\mu)-e_\varepsilon^\star
    =
    \mathcal H_{s}^{\varepsilon}(\mu)-\inf_{\nu\in\mathcal P(\Omega)}\mathcal H_{s}^{\varepsilon}(\nu).
    \]
    
    Finally, let $\varepsilon\downarrow 0$. Since $g_{s}^{\varepsilon}\uparrow g_s$, we have $\mathcal H_{s}^{\varepsilon}(\mu)\uparrow \mathcal H_s(\mu)$
    pointwise (possibly $+\infty$), and
    \[
    e_\varepsilon^\star=\inf_{\mu}\mathcal H_{s}^{\varepsilon}(\mu)\uparrow \inf_{\mu}\mathcal H_s(\mu)\defeq e^\star.
    \]
    Thus $\mathcal J^{\varepsilon}_s(\mu)\uparrow \mathcal J_s(\mu)$ pointwise, where
    \[
    \mathcal J_s(\mu)
    \;\defeq\;
    \mathcal H_s(\mu)-e^\star
    =
    \mathcal H_s(\mu)-\inf_{\nu\in\mathcal P(\Omega)}\mathcal H_s(\nu),
    \]
    which is the claimed rate function in the regime $\beta_N/N\to\infty$.

    \medskip
    \noindent\subsubsection{goodness of rate function.}
    In the first regime, $\Ent(\cdot\mid \ell)$ has compact level sets (Sanov rate is good),
    and $\mathcal E_s$ is lower semicontinuous, so $\mathcal E_s+\Ent(\cdot\mid\ell)$ has compact sublevel sets,
    hence $\mathcal J_s$ is good.
    In the second regime, $\mathcal E_s$ is lower semicontinuous and the underlying space is bounded,
    so the sublevel sets of $\mathcal E_s$ are tight; together with lower semicontinuity this yields goodness.


\end{document}